\pdfoutput=1

\documentclass[11pt]{article}
\usepackage[table]{xcolor}

\usepackage[final]{acl}

\usepackage{bbm}
\usepackage[]{todonotes} %
\makeatletter
\newcommand*\iftodonotes{\if@todonotes@disabled\expandafter\@secondoftwo\else\expandafter\@firstoftwo\fi}  %
\makeatother

\definecolor{dandelion}{HTML}{FFD464}

\definecolor{bittersweet}{HTML}{C04F17}

\definecolor{mintgreen}{RGB}{152, 255, 152}

\usepackage{tipa}

\usepackage{times}
\usepackage{latexsym}
\usepackage{fontawesome}
\usepackage{siunitx}
\usepackage{xspace}
\usepackage[inline]{enumitem}
\usepackage[multiple]{footmisc}
\usepackage{mleftright}
\usepackage{emoji}
\usepackage{cuted,tcolorbox}
\usepackage{thm-restate}
\usepackage{xfrac}

\usepackage{xurl}
\usepackage{multirow}
\usepackage{listings}
\usepackage{scale}
\newcommand{\centerBox}{{\raisebox{-1pt}{$\Box$}}}

\def\permtilde{{\widetilde{\perm}}}

\usepackage[T1]{fontenc}

\usepackage[utf8]{inputenc}

\usepackage{microtype}

\usepackage{inconsolata}

\usepackage{graphicx}

\usepackage[noend]{algpseudocode}
\usepackage{algorithm}
\usepackage{amsfonts}
\usepackage{amsmath}
\usepackage{amssymb}
\usepackage{bm}
\usepackage{cleveref}
\usepackage{xspace}
\usepackage{amsthm}
\usepackage{mathtools}
\usepackage{relsize}
\usepackage{caption,subcaption,booktabs}
\usepackage{import}
\usepackage{hyphenat}
\usepackage{booktabs}
\usepackage{graphicx}
\usepackage{tikz}
\usepackage{tabularx}
\usetikzlibrary{positioning}
\usetikzlibrary{backgrounds}
\usetikzlibrary{calc}

\newcommand{\answercolor}{\color{MyTawny}}
\newcommand{\contextcolor}{\color{MyGreen}}
\newcommand{\goldcontextcolor}{\color{MyGold}}
\newcommand{\questioncolor}{\color{MyBlue}}

\newcommand{\textanswer}{{\answercolor \text{answer}}}
\newcommand{\textcontext}{{\contextcolor \text{context}}}
\newcommand{\textquestion}{{\questioncolor \text{question}}}

\newcommand{\modelname}[1]{{\tt{#1}}}

\newcommand{\indicator}[1]{\mathbbm{1}\mleft\{#1\mright\}}
\newcommand{\mychar}[0]{\sigma}

\newcommand{\charseq}{{{\boldsymbol{\mychar}}}}

\DeclareMathOperator*{\argmax}{arg\,max}

\crefname{section}{Section}{Sections}
\Crefname{section}{Section}{Sections}
\crefname{table}{Table}{Tables}
\crefname{figure}{Figure}{Figures}
\crefname{algorithm}{Algorithm}{Algorithms}
\crefname{equation}{Eq.}{Eqs.}
\crefname{appendix}{Appendix}{Appendices}
\crefname{thm}{Theorem}{Theorems}
\crefname{prop}{Proposition}{Propositions}
\crefname{cor}{Corollary}{Corollaries}
\crefname{observation}{Observation}{Observations}
\crefname{assumption}{Assumption}{Assumptions}
\crefname{hypothesis}{Hypothesis}{Hypotheses}
\crefformat{section}{\S#2#1#3}

\crefrangeformat{equation}{Eqs. #3(#1)#4--#5(#2)#6}
\crefrangeformat{figure}{Figures #1--#2}

\theoremstyle{plain}

\newtheorem{example}{Example}[section]

\newtheorem{assumption}{Assumption}[section]
\newtheorem{hypothesis}{Hypothesis}[section]

\newcommand{\defn}[1]{{\textbf{#1}}}

\newcommand{\defequals}{\triangleq}

\usetikzlibrary{shapes.misc}

\newcommand{\kleene}[1]{{#1^*}}
\newcommand{\STR}{{\kleene{\Sigma}}}

\def\calD{{\mathcal{D}}}

\def\boldq{{\boldsymbol{q}\xspace}}
\def\boldc{{\boldsymbol{c}\xspace}}

\definecolor{MyGold}{HTML}{D8D000}
\definecolor{MyTawny}{HTML}{d55e00}
\definecolor{MyGreen}{HTML}{029e73}
\definecolor{MyBlue}{HTML}{0173b2}
\definecolor{MyOrange}{HTML}{de8f05}
\definecolor{MyBronze}{HTML}{ca9161}
\definecolor{MySilver}{HTML}{949494}
\definecolor{MyKerria}{HTML}{F8B500}
\definecolor{MyPurple}{HTML}{952b60}

\newcommand{\exampletext}[1]{{\contextcolor \small \textit{``#1''}}}

\def\corpus{{\mathcal{C}}}
\def\permsigma{{\color{MyOrange} \sigma}}
\def\perm{{\color{MyOrange} \pi}}
\def\bcdot{\mathbin{\boldsymbol{\cdot}}}
\newcommand{\documentd}{{\contextcolor \boldsymbol{d}}}
\newcommand{\question}{{\questioncolor \boldq}}
\newcommand{\context}{{\contextcolor \boldc}}
\newcommand{\contextDpi}{{\contextcolor \boldc_{\docset}(\perm)}}

\newcommand{\docset}{{\contextcolor \calD}}

\newcommand{\metric}{{g}}
\newcommand{\pmi}{{\textnormal{PMI}}}
\newcommand{\alphabet}{\Sigma}

\newcommand{\lm}{p} %
\newcommand{\lmprefix}{\overrightarrow{p}}

\newcommand{\answer}{{\answercolor \boldsymbol{a}}}
\newcommand{\decodeanswer}{{\answercolor \boldsymbol{\widetilde{a}}}}
\newcommand{\decodeanswerpi}{\decodeanswer_{\perm}}

\DeclarePairedDelimiterX{\infdivx}[2]{(}{)}{#1\;\delimsize\Vert\;#2}

\definecolor{red_fig}{HTML}{D95847}
\definecolor{blue_fig}{HTML}{5D7CE6}

\definecolor{skipcolor}{HTML}{CC78BC}
\definecolor{valuecolor}{HTML}{029E73}
\definecolor{querycolor}{HTML}{DE8F05}
\definecolor{keycolor}{HTML}{0173B2}

\renewcommand{\Pr}[2][]{\mathop{\mathbb{P}}_{\substack{#1}}\mleft(#2\mright)}

\title{Pointwise Mutual Information as a Performance Gauge for Retrieval-Augmented Generation}

\newcommand{\rug}{\emoji[icon]{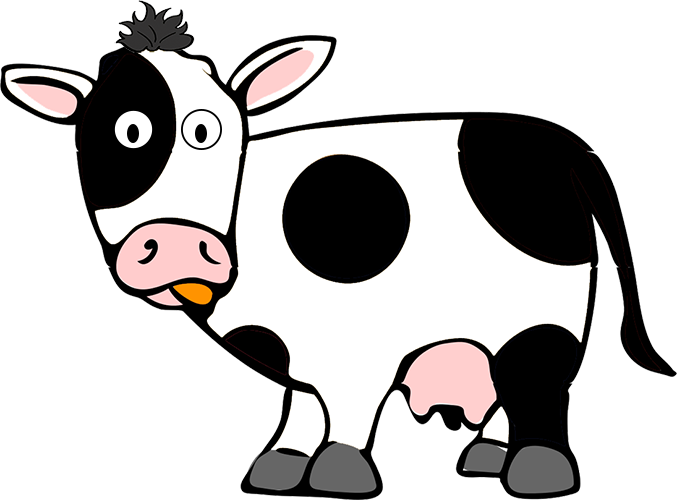}}
\newcommand{\ethz}{\emoji[icon]{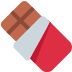}}
\newcommand{\toronto}{\emoji[icon]{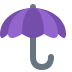}}

\author{
  Tianyu Liu$^{\ethz,}$\thanks{The first two authors contributed equally.} \quad Jirui Qi$^{\rug,}$$^{*}$ \quad Paul He$^{\toronto,}$\thanks{Work performed while at ETH Zürich.} \\
  {\bf Arianna Bisazza}$^{\rug}$ \quad {\bf Mrinmaya Sachan}$^{\ethz}$ \quad {\bf Ryan Cotterell}$^{\ethz}$ \\
  $^{\ethz}$ETH Zürich ~\;~ $^{\rug}$CLCG, University of Groningen\;~ $^{\toronto}$University of Toronto \\
  \texttt{\{\href{mailto:tianyu.liu@inf.ethz.ch}{tianyu.liu}, \href{mailto:mrinmaya.sachan@inf.ethz.ch}{mrinmaya.sachan},\href{mailto:ryan.cotterell@inf.ethz.ch}{ryan.cotterell}\}@inf.ethz.ch} \\
  \texttt{\{\href{mailto:j.qi@rug.nl}{j.qi}, \href{mailto:a.bisazza@rug.nl}{a.bisazza}\}@rug.nl, \href{mailto:hepaul@cs.toronto.edu}{hepaul}@cs.toronto.edu}
}

\begin{document}
\maketitle
\begin{abstract}
Recent work suggests that large language models enhanced with retrieval-augmented generation are easily influenced by the order, in which the retrieved documents are presented to the model when solving tasks such as question answering (QA).
However, there is no method to date that exploits this phenomenon to improve generation.
We fill this gap.
In this study, we show that the pointwise mutual information between a context and a question is an effective gauge for language model performance.
Importantly, this gauge does not depend on knowing the answer to the question \textit{a priori}.
Through experiments on two question-answering datasets and a variety of large language models, we find evidence for an empirical correlation between answer accuracy and pointwise mutual information.
Additionally, we propose two methods that use the pointwise mutual information between a document and a question as a gauge for selecting and constructing prompts that lead to better performance, whose effectiveness we demonstrate through experimentation.\footnote{Our code is available at \url{https://github.com/lyutyuh/poptimizer}.}

\end{abstract}

\section{Introduction}\label{sec:introduction}

\begin{figure}
\centering
\includegraphics[trim={0 20pt 0 20pt},clip,width=0.9\linewidth]{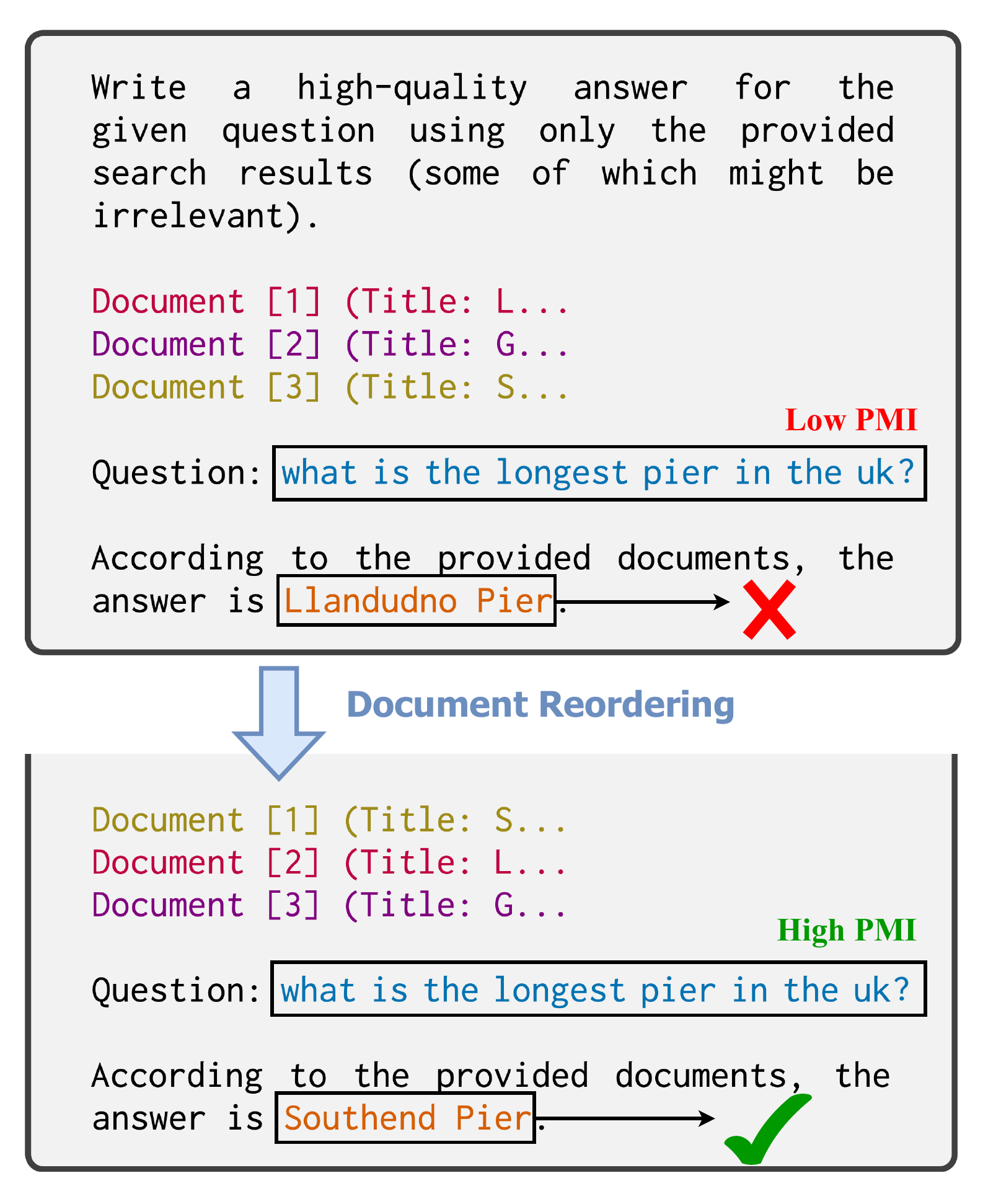}
\caption{For the \emph{same} question, a permutation of documents with a higher $\pmi(\question, \contextDpi)$ tends to lead to a better answer. 
}
\label{fig:motivation}
\end{figure}

Prompt design is an important factor when applying language models (LMs) to downstream tasks, including LMs that make use of retrieval-augmented generation \citep[RAG;][]{rag-lewis-2020}. Well-constructed prompts can improve LMs' answers to user-input questions and help generate responses that better align with user expectations \citep[][\textit{inter alia}]{gao-etal-2021-making, izacard2022atlas, liu-etal-2024-lost, schulhoff2024prompt, ma2024crafting}.\looseness=-1

Under the RAG framework, a prompt typically consists of three components. 
First, an instruction provides a textual description of the overall task and general guidance for the language model. Second, a specific 
{\textquestion} encodes the precise task or query the model should perform. 
Third, a {\textcontext} encodes a set of documents retrieved from an external source by a retriever \citep{karpukhin-etal-2020-dense, ni-etal-2022-large}.
Then, an {\textanswer} is sampled from the language model.
Previous work has explored various empirical approaches to prompt engineering, including the manual design of prompts that mimic human reasoning \cite{wei2023chainofthoughtpromptingelicitsreasoning,yao2023treethoughtsdeliberateproblem}. Recently, \citet{liu-etal-2024-lost} demonstrated that language model performance is significantly influenced by the \emph{order} of retrieved documents that comprise the context. 
Specifically, QA accuracy peaks when the gold document\footnote{In factual QA tasks, the document containing the ground truth answer is referred to as the gold document.} is positioned at the beginning or end of the context.
Although extensive experimental evidence was adduced to validate this phenomenon \cite{liu-etal-2024-lost}, the underlying mechanisms remain poorly understood. This gap in understanding limits the applicability of these findings in the design and optimization of prompts for real-world applications.\looseness=-1

While \citeposs{liu-etal-2024-lost} are interesting, choosing the optimal permutation of the documents requires knowledge of the $\textanswer$, and, thus, cannot be directly used to improve RAG.
In this article, we develop a proxy for the optimal permutation: We show that the pointwise mutual information between the {\textquestion} and the {\textcontext} under an LM acts as a useful proxy in determining the optimal permutation.
To our knowledge, ours is the first to present in-depth analyses of the relation between question likelihood and model performance under the RAG framework.

Our findings in this paper are summarized in the following list: \looseness=-1
\begin{itemize}[noitemsep, nolistsep]
    \item We show that the pointwise mutual information between the {\textquestion} and the {\textcontext} positively correlates with answer accuracy at the corpus level on NQ-Open \cite{kwiatkowski-etal-2019-natural,lee-etal-2019-latent} and ELI5 \cite{fan-etal-2019-eli5}.\looseness=-1
    \item Given a question and a fixed set of documents, we demonstrate a strong correlation between the position of the gold document, the PMI between the {\textquestion} and the {\textcontext}, and QA accuracy.\looseness=-1 
     \item We validate the effectiveness of using question likelihood as a gauge for prompt optimization and demonstrate that likelihood-based prompt optimization is a promising direction for future study.\looseness=-1
\end{itemize}

\section{Setting the Stage}\label{sec:setting-the-stage}

\subsection{Language Modeling and RAG} 

\paragraph{Language Modeling Background.}
Let $\alphabet$ be an \defn{alphabet}, i.e., a finite, non-empty set of \defn{tokens}.
A \defn{language model} $\lm$ is a distribution over $\STR$, the set of all strings with tokens drawn from $\alphabet$.
Let $Y$ be a $\STR$-valued random variable distributed according to $\lm$ and $\charseq \in \STR$. 
We define the \defn{prefix probability}\footnote{See 
\citet{vieira2024languagemodelstokenslanguage} for a more in-depth discussion.} $\lmprefix(\charseq)$ as the probability that $Y$ has $\charseq$ as a prefix:
\begin{subequations}
\begin{align}\label{eq:prefix-prob}
\lmprefix(\charseq) &\defequals 
\Pr[Y \sim \lm]{ Y \succeq \charseq} 
\\ &= \sum_{ \charseq' \in \STR } \indicator{\charseq' \succeq \charseq} \,  \lm(\charseq')
\end{align}
\end{subequations}
The conditional prefix probability $\lmprefix(\charseq' \mid \charseq) = \frac{\lmprefix(\charseq \bcdot \charseq')}{\lmprefix(\charseq')}$ tells us how certain the model is that $\charseq'$ naturally follows from its preceding string $\charseq$.
Finally, we define an \defn{infix probability}, i.e., the  probability of generating a string that contains $\charseq \centerBox \charseq';$ where as $\centerBox$ is a gap, as follows
\begin{subequations}
\begin{align}
&\lmprefix(\charseq  \centerBox \charseq'') \defequals 
\Pr[Y \sim \lm]{ Y \succeq \charseq \centerBox \charseq''} 
\\ & \,\, = \sum_{\charseq''' \in \STR } \sum_{\charseq' \in \STR } \indicator{ \charseq''' \succeq \charseq \charseq' \charseq'' } \,  \lm(\charseq''')
\end{align}
\end{subequations}

\paragraph{Retrieval-augmented Generation.}
Modern language models are often used to perform question-answering tasks. 
When solving such a task with a language model, string encoding the question {\questioncolor question} $\question \in \STR$ is given to the model.
We assume each question $\question$ has a unique correct {\answercolor answer} which we will denote $\answer$.
This is, of course, a simplifying assumption, but it does jibe with how question-answering is typically evaluated.
We will adorn a $\widetilde{\cdot}$, e.g., $\decodeanswer$, to indicate an {\answercolor answer} generated from $\lmprefix(\bcdot \mid \question)$ that may or may not be correct.
Generating $\decodeanswer$ from $\lmprefix(\bcdot \mid \question)$ may be done using either a deterministic method, e.g., beam search, or a stochastic method, e.g., ancestral sampling.\footnote{In this study, we consider greedy decoding, i.e., beam search with a beam of size 1.}
In RAG, the model is additionally given a set of documents $\docset = \{\documentd_k\}_{k=1}^{K}$, where $\documentd_k \in \STR$, and a permutation of the documents $\perm\colon \{ 1,\cdots,K\} \to \{1,\cdots,K \}$.
Given $\docset$ and $\perm$, a {\contextcolor context} $\context$ is constructed by concatenating the documents in the order defined by $\perm$, i.e., $\contextDpi \defequals \documentd_{\perm(1)}\bcdot  \cdots \bcdot\documentd_{\perm(K)}$. 
Then, we generate an answer from $\lmprefix(\bcdot \mid \context \bcdot \question)$.
We provide an example below.\looseness=-1
\begin{example}
    Consider $\docset = \{$\exampletext{Llandudno Pier is a Grade II* listed pier\textellipsis}, \exampletext{Garth Pier is a Grade II listed structure\textellipsis}, \exampletext{Southend Pier is a\textellipsis}$\}$, and $\perm(1) = 2, \perm(2)=1, \perm(3)=3$. 
    We have $\contextDpi = $ \exampletext{Garth Pier is\textellipsis Llandudno Pier is a Grade II* listed pier\textellipsis Southend Pier is\textellipsis}.
\end{example} 
\noindent Let $\decodeanswerpi$ denote an answer generated from $\lmprefix(\bcdot \mid \contextDpi\bcdot\question)$. 
To evaluate the quality of $\decodeanswerpi$, we define an evaluation metric $\metric(\decodeanswerpi, \answer)$.
In addition, we assume the ground truth answer $\answer$ to be unique for a \textquestion--\textcontext{} pair $(\question, \context)$.

\paragraph{Pointwise Mutual Information.} 
In RAG question answering, we consider the following \defn{pointwise mutual information} 
\begin{equation}\label{eq:pmi}
\pmi(\question, \context) \defequals \log \frac{\lmprefix(\question \mid \context)}{\lmprefix(\question)}
\end{equation}
between $\question$ and $\context$, where $\context= \documentd_{\perm(1)}\bcdot  \cdots \bcdot\documentd_{\perm(K)}$.
In other words, \Cref{eq:pmi} measures the degree of association of $\question$ with $\context$. 

\begin{figure}
\centering
\includegraphics[trim={0mm 70mm 0mm 70mm},clip,width=0.99\linewidth]{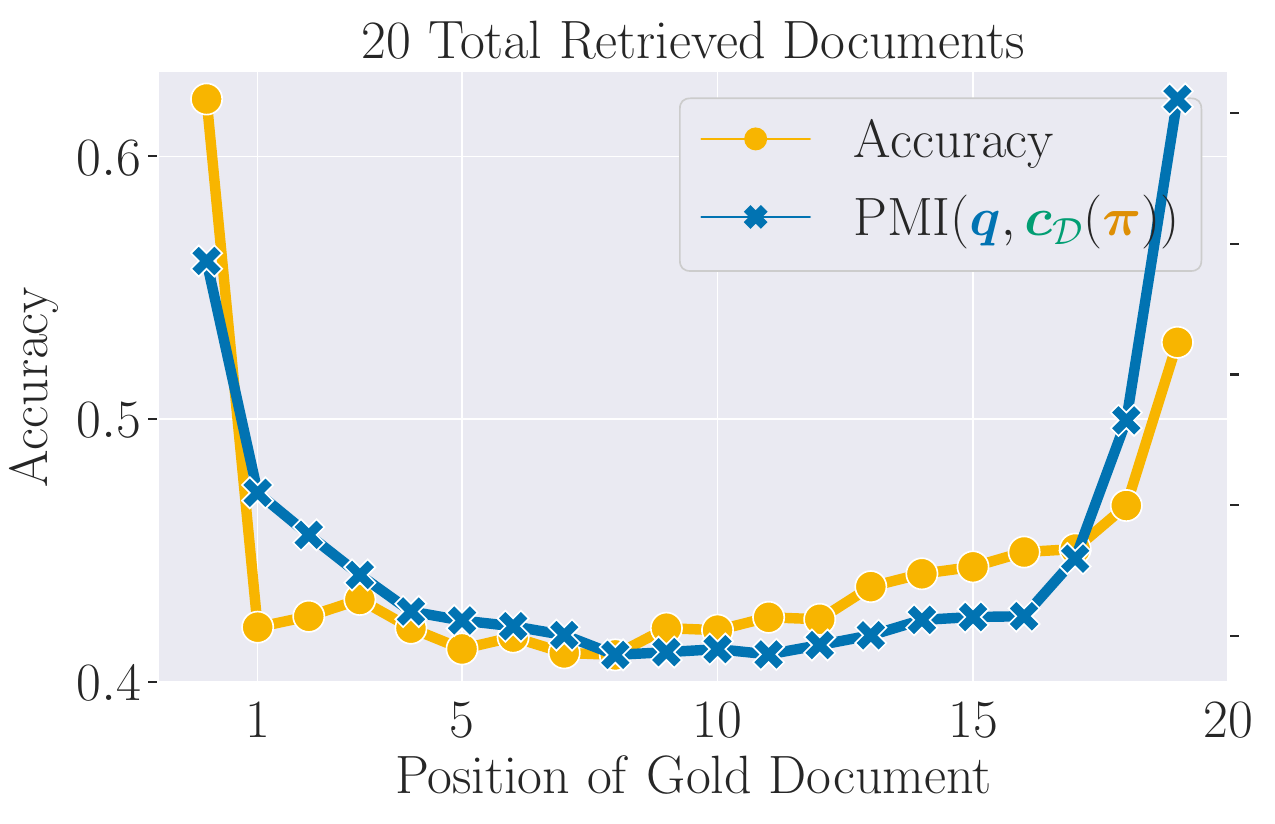}
\caption{
We observe that the PMI and QA accuracy trace a U-shaped curve---nearly in lockstep---as the gold document position within the context changes.
The result is computed with \modelname{LLaMA-3-8B}.
}
\label{fig:intro-p-question}
\end{figure}
\subsection{A Concrete Hypothesis}
Returning to the central goal of this paper, i.e., trying to find a proxy that helps determine the optimal permutation of the documents for RAG, we hypothesize that, given a question $\question$, a set of documents $\docset$, and the ground truth answer $\answer$, the pointwise mutual information $\pmi(\question, \contextDpi)$ correlates with $\log \frac{\lmprefix(\answer \mid \question \bcdot \contextDpi)}{1-\lmprefix(\answer \mid \question \bcdot \contextDpi)}$, the log odds ratio, and can be deemed a gauge for the expected accuracy of the generated answer.
In symbols, our hypothesis is as follows.\looseness=-1
\begin{hypothesis}\label{hyp:hypothesis}
    In RAG question answering, for a fixed question $\question$, a set of documents $\docset$ permuted by $\perm$, and the ground truth answer $\answer$, we have the following relation between $\pmi(\question, \contextDpi)$ and $\lmprefix(\answer \mid \question \bcdot \contextDpi)$
    \begin{equation}
    \begin{aligned}
    \pmi(\question&, \contextDpi) \\
    &= a \log \frac{\lmprefix(\answer \mid \question \bcdot \contextDpi)}{1-\lmprefix(\answer \mid \question \bcdot \contextDpi)} + b
\end{aligned}
    \end{equation}
    for constants $a \in \mathbb{R}_{>0}$, $ b \in \mathbb{R}$.
\end{hypothesis} 

\subsection{A Bit of Analysis}

\def\baranswer{\bar{\answer}}

In words, \Cref{hyp:hypothesis} says that when $\text{PMI}(\question, \contextDpi)$ is high, we expect the LM better respond to the question $\question$ with higher accuracy, and, moreover, this relationship is affine.
Although this is empirically true in many cases \cite{gonen-etal-2023-demystifying}, we offer an assumption that enables a derivation of this property. 
\begin{assumption}\label{ass:assumption}
    For all \textquestion--\textcontext{} pairs $(\question, \context)$, 
    let $\answer$ be the correct answer, then we have
    \begin{subequations}
        \begin{align}
             \lmprefix(\question \mid \contextDpi \centerBox \answer) &=   \lmprefix(\question \mid \contextDpi) \label{eq:correct} \\
            \lmprefix(\question \mid \contextDpi \centerBox \baranswer) &=   \lmprefix(\question) \label{eq:incorrect}
        \end{align}
    \end{subequations} 
    for any $\baranswer\in\STR$ such that $\baranswer \not\preceq \answer$.\footnote{\Cref{eq:correct} and \Cref{eq:incorrect} can be seen as a form of context-specific conditional independence \citep{boutilier+al.1996}.}
\end{assumption}
We now give a brief qualitative justification of
\Cref{ass:assumption}.
Conditioned on the event that the model \emph{incorrectly} answers the {\questioncolor question} given the {\contextcolor context}, \cref{eq:incorrect}  says that the question $\question$ is not dependent on the provided \textcontext.
Because, in RAG, we assume the correct \textanswer{} is given to the model in the \textcontext{} and the model's job is to retrieve it, our assumption corresponds to the notion that an incorrect response by the model should \emph{not} be influenced by the \textcontext.
\Cref{eq:correct} corresponds to the notion that since the correct \textanswer{} is already contained in the \textcontext, conditioning on the correct answer \textanswer{} does not provide any new information to generating $\question$.

\begin{restatable}{proposition}{propdistribution}\label{prop:distribution}
Under assumptions given in \cref{ass:assumption}, we have
\begin{equation}
\begin{aligned}
    \log &\frac{\lmprefix(\answer \mid \question \bcdot \contextDpi)}{1-\lmprefix(\answer \mid \question \bcdot \contextDpi)} \\
    &\qquad = \pmi(\question,\contextDpi) + C(\answer, \contextDpi)
\end{aligned}
\end{equation}
for an {\answercolor answer}-dependent constant $C(\answer, \contextDpi)$.
\end{restatable}

\begin{proof}
See \Cref{app:proof}.
\end{proof}
\noindent In other words, the pointwise mutual information $\pmi(\question, \contextDpi)$ is equivalent up to an additive constant to 
the log odds ratio $\log\frac{\lmprefix(\answer \mid \question \bcdot \contextDpi)}{1-\lmprefix(\answer \mid \question \bcdot \contextDpi)}$.

\paragraph{Foreshadowing the Results.}
In the empirical portion of this paper, we test \Cref{hyp:hypothesis} through experiments on two QA benchmarks—NQ-Open and ELI5—using a range of state-of-the-art open LMs, including \modelname{LLaMA-2}, \modelname{LLaMA-3}, \modelname{LLaMA-3.1}, \modelname{Mistral-v0.3}, and \modelname{MPT}.
Our findings demonstrate that, as the position of relevant information within the input context $\context$ varies, the pointwise mutual information, $\pmi(\question,\contextDpi)$ and expected answer accuracy $\lmprefix(\answer \mid \question \bcdot \contextDpi)$ vary in tandem, i.e., they are strongly correlated.
This correlation is illustrated in \cref{fig:intro-p-question}. Specifically, LMs tend to provide better responses to questions where the documents in the context are permuted so as to have higher $\pmi(\question,\contextDpi)$.
These results suggest that PMI serves both as a \emph{performance gauge} and as a strong indicator of the position of task-relevant information within the input context. Building on this insight, we propose a direction for prompt optimization through two specific methods. 
The first selects a permutation $\perm$ of the documents that maximizes $\pmi(\question,\contextDpi)$ to construct the context $\context$. 
The second builds on the findings of \citet{liu-etal-2024-lost} that the curve traced by permuting the position of the gold document results in a U-shaped curve.
We exploit this finding to develop an efficient prompt ordering algorithm. 
Further experimentation demonstrates that our methods enhance answer accuracy across both datasets for instruction-tuned and base models alike, with the second approach achieving even greater gains.

\section{PMI Correlates with Performance}
As discussed in \cref{sec:setting-the-stage}, 
our first goal is to determine how $\pmi(\question, \contextDpi)$ changes as a function of the permutation $\perm$ of the documents $\docset$.
Due to \Cref{hyp:hypothesis}, we expect a strong correlation between $\pmi(\question, \contextDpi)$ and expected answer accuracy. 

\begin{figure*}
    \centering
    \includegraphics[clip,width=0.95\linewidth]{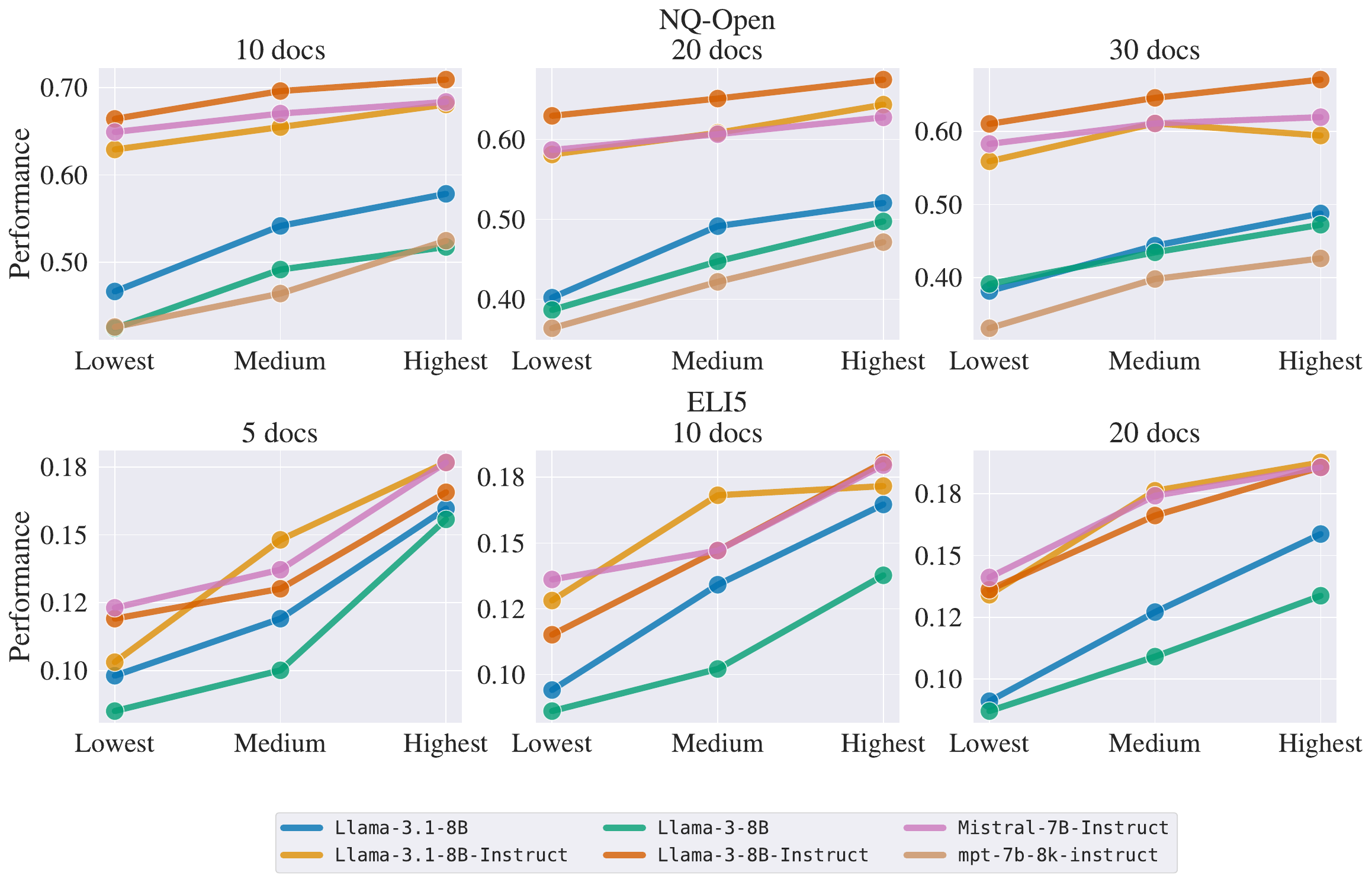}
    \caption{Corpus-level correlation between $ \pmi(\question, \contextDpi)$ and answer accuracy on NQ-Open and ELI5.}
    \label{fig:double-corpus-level-correlation}
    \vspace{-5pt}
\end{figure*}

\subsection{Experimental Setup}
\label{sec:setups}
\paragraph{Datasets.}
We run experiments on two question-answering datasets, namely NQ-Open and ELI5. 
Details of the datasets are given in \cref{app:datasets}.
Let $\corpus \defequals \{(\question_m, \docset_m, \answer_m)\}_{m=1}^{M}$ be a dataset of triples, where each $\answer_m$ represents the ground truth answer to $\question_m$.\looseness=-1

\paragraph{Empirical Metrics for LM evaluation.} \label{sec:answer-acc-metrics}
In practice, LM performance is often evaluated with rule-based empirical metrics, denoted $\metric$, such as accuracy, instead of the conditional likelihood $\lmprefix(\answer \mid \question \bcdot \contextDpi)$. 
Although mathematically quantifying the relation between $\metric$ and $\pmi$ is difficult, we contend that they positively correlate due to recent progress on language model calibration \cite{zhao2023calibrating}, i.e., the alignment between $\lmprefix(\answer \mid \question \bcdot \contextDpi)$ and $\metric(\decodeanswer,\answer)$.
On {NQ-Open}, the ground truth answer for each question is either a word or a short phrase.
The accuracy is $1$ when the LM response contains the correct answer as a substring; otherwise, the accuracy is $0$. %
Following \citet{liu-etal-2024-lost}, we compute the model's average accuracy over the entire dataset. 
On {ELI5}, the correct answer for each question comprises three sub-claims, and a correct answer is expected to include all of these sub-claims.
Examples are illustrated in \cref{{app:illustration_evaluation_metric}}.\
We follow \citet{gao-etal-2023-enabling} and take the recall rate of sub-claims to be the evaluation metric, which takes value from $\{0, \sfrac{1}{3}, \sfrac{2}{3}, 1 \}$.
The \modelname{TRUE} model,\footnote{\url{https://huggingface.co/google/t5_xxl_true_nli_mixture}} a \modelname{T5-XXL} model fine-tuned on natural language inference (NLI) tasks, is used to automatically evaluate whether a response entails a sub-claim. \looseness=-1

\paragraph{Language Model Settings.}
Most state-of-the-art closed LMs, such as OpenAI’s \modelname{ChatGPT} and Anthropic’s \modelname{Claude}, do not provide direct access to the likelihood of either input or output tokens. Thus, we select leading open LMs for our experiments, focusing on three families: \modelname{LLaMA-2}, \modelname{LLaMA-3} \cite{touvron2023llamaopenefficientfoundation}, and \modelname{Mistral-v0.3} \cite{jiang2023mistral7b}. We also evaluate \modelname{MPT} on NQ-Open.\footnote{In our preliminary experiments, \modelname{MPT} fails to generate sufficiently long responses on ELI5, resulting in performance that is not directly comparable to other LMs.} Following the settings of \citet{liu-etal-2024-lost}, we adopt greedy decoding for all models when generating responses. We set the maximum number of decoded tokens to $100$ on NQ-Open and $300$ on ELI5.

\paragraph{Prompt Templates.} 
We follow the suggested usage and prompt formatting instructions of each LM we use.
For chat and instruction-tuned models, we present the context and query to the LM in the role of \texttt{user}, and elicit the response from LMs in the role of \texttt{assistant}.
For base models, we elicit responses from LMs as sentence completion. \looseness=-1

\subsection{Technical Interlude: Sets of Permutations}\label{sec:sets-of-permutation}
In many of our experiments, we would like to take a sum or a max over all permutations of $K$ items, i.e., 
take a sum or max over the symmetric group $\mathbb{S}_K$.
However, $|\mathbb{S}_K| = K!$, which grows too large to enumerate efficiently. 
To cope with the size of $\mathbb{S}_K$, in this paper, we perform computations over a subset of $\mathbb{S}_K$.
Specifically, starting a user-specified permutation $\perm$, we consider the cyclic group generated by $(\perm)$ where the group operation is functional composition, as is standard.
Let $\permsigma=(1,2,\cdots,K)$ be a shifting permutation.
It is easy to see that $|(\perm)| = K$, and
the $k^{\text{th}}$ element of $(\perm)$ is given by $ \permtilde_k \defequals \permsigma^{k-1}\circ\perm = \underbrace{\permsigma \circ \cdots \circ \permsigma }_{(k-1)\text{ times}}\circ \perm$,\footnote{Here, $\circ$ denotes function composition, the standard product of permutations.} or, equivalently we have
\begin{equation}
    \permtilde_k(i) = (i + k - 1) \mod K,
\end{equation}
\begin{example}
    Given a permutation $\perm = (1,2,3)$. 
    The cyclic group $(\perm)$ generated by $\perm$ is equal to $\{\permtilde_1, \permtilde_2, \permtilde_3\}$ where $\permtilde_1 = (1,2,3)$, $\permtilde_2 = (2,3,1)$, and $\permtilde_3 = (3,1,2)$.
\end{example}

\subsection{Results} \label{sec:results}
We now discuss our empirical findings. 

\paragraph{Corpus-Level Correlation.}
As our first evaluation metric, we consider a corpus-level correlation.
For each $\question_{m}, \docset_{m}$ in a corpus $\corpus$, we compute the average PMI for the $m^{\text{th}}$ instance as follows
\begin{equation}
    \rho_m \defequals \frac{1}{K}\sum_{k=1}^K  \pmi(\question_{m}, \context_{\docset_{m}}(\permtilde_{k}))
\end{equation}
We then bin the elements of $\{\rho_m\}_{m=1}^M$ into three bins according to which tertile they fall it when $\{\rho_m\}_{m=1}^M$ are arranged into a histogram.
Then, we compute the average sub-claim recall rate (ELI5) and accuracy (NQ-Open) for each bin.
Our results, shown in \cref{fig:double-corpus-level-correlation}, demonstrate that LMs tend to perform better on the prompts with a higher $\pmi(\question, \contextDpi)$ compared to those with lower $\pmi(\question, \contextDpi)$. 

\paragraph{Instance-Level Correlation.} \label{sec:instance-correlation}
We further analyze the instance-level correlation between $\pmi(\question, \contextDpi)$ and accuracy by varying $\textcontext$ while keeping the $\textquestion$ fixed. %
In symbols, we compute
\begin{equation}
    \eta_k \defequals \frac{1}{M}\sum_{m=1}^M  \pmi(\question_{m}, \context_{\docset_{m}}(\permtilde_{k}))
\end{equation}
where $\permtilde_k$ is the permutation in which the $k$-th document $\documentd_k$ contains relevant information.
We then plot the curve of $\{\eta_k\}_{k=1}^K$, to see how $\pmi$ is affected by the position of a relevant document within a context.

\paragraph{Revisiting \citet{liu-etal-2024-lost}.}\label{sec:nq-open-instance-level}
We now revisit the findings of \citet{liu-etal-2024-lost}, who observed a drop in \emph{QA accuracy} when the gold document is positioned within the middle of $\context$.
We first experiment on NQ-Open by varying the position of the gold document\footnote{In NQ-Open, exactly \emph{one} retrieved document is marked as the gold document for each question.} in $\context$. 
The set of retrieved documents and the order of non-gold documents remain the same.
As the gold document is placed in different positions in $\context$, we find that both $\pmi(\question, \contextDpi)$ and QA accuracy fluctuate---nearly in lockstep.
To further explore this correlation between $\pmi(\question,\context)$ and QA accuracy, we calculate the expected accuracy with the prompt of the highest and lowest $\pmi(\question, \contextDpi)$. %
Results are given in \cref{tab:instance-level-correlation-nqopen}, showing that LMs perform better when the document order in the prompt leads to the highest $\pmi(\question, \contextDpi)$; while the prompt with the lowest $\pmi(\question, \contextDpi)$ results in inferior performance.

\begin{figure*}[!t]
    \centering
    \begin{subfigure}{0.49\textwidth}
        \centering
        \includegraphics[trim={0mm 60mm 0mm 70mm},clip,width=\linewidth]{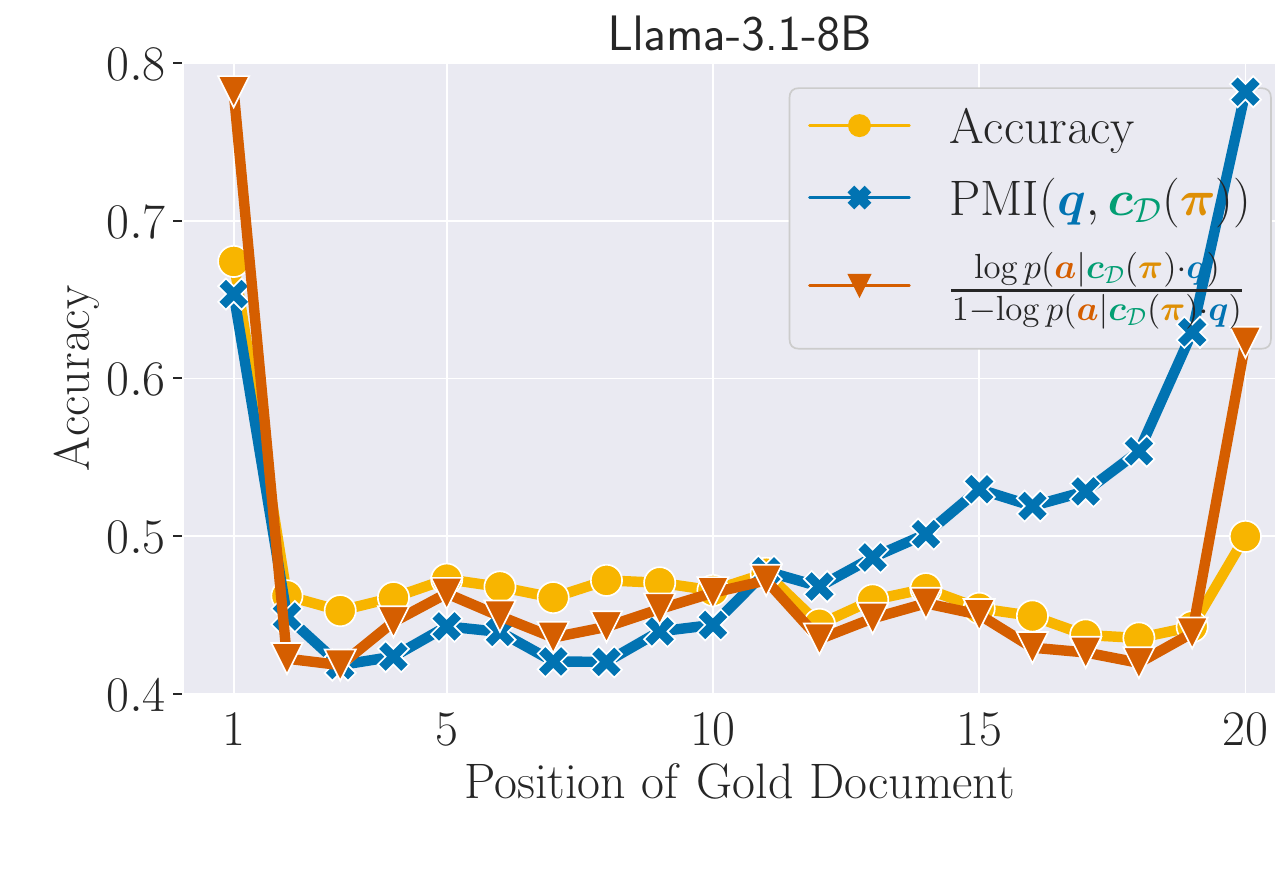}
        \caption{}
        \label{fig:20docs-llama-3.1}
    \end{subfigure}
    \begin{subfigure}{0.49\textwidth}
        \centering
        \includegraphics[trim={0mm 60mm 0mm 70mm},clip,width=\linewidth]{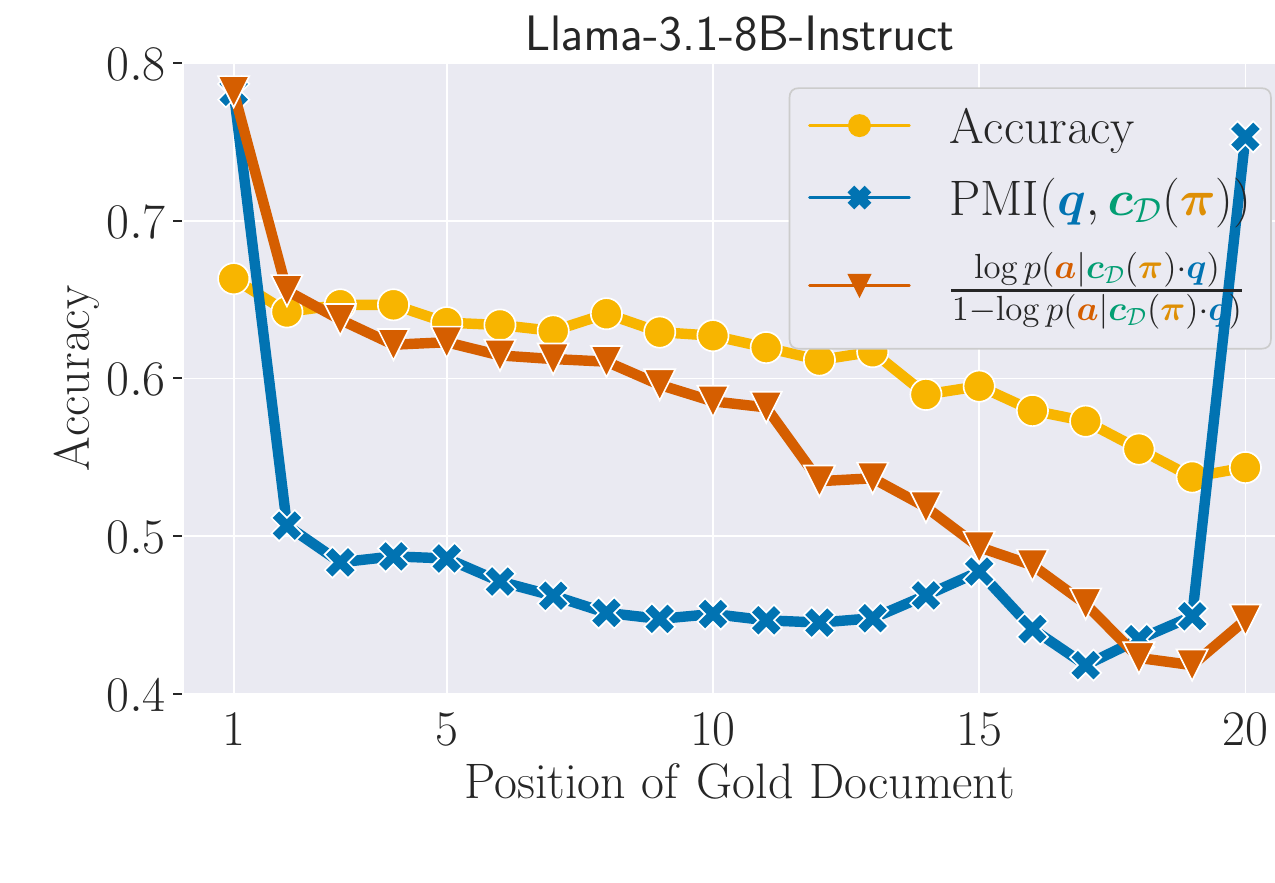}
        \caption{}
        \label{fig:20docs-llama-3.1-inst}
    \end{subfigure}
    \caption{QA accuracy, PMI, and log odds ratio of answer likelihood on 20 docs evaluated on \modelname{LLaMA-3.1-8B} and \modelname{LLaMA-3.1-8B-Instruct}.}
    \label{fig:20docs-llama-3.1-all}
    \vspace{-5pt}
\end{figure*}

\begin{table*}
\centering
\small
\resizebox{0.99\linewidth}{!}{
\begin{tabular}{cc@{~} cccccc @{}}
\toprule
\#Doc & $ \pmi(\question, \contextDpi)$ & \modelname{Mistral-7B-Inst} & \modelname{LLaMA-3-8B} & \modelname{LLaMA-3.1-8B} & \modelname{LLaMA-3-8B-Inst} & \modelname{LLaMA-3.1-8B-Inst} & \modelname{MPT-7B-8K-Inst}\\
\midrule
\\[-12.5pt]
\rowcolor{black!10} \multicolumn{8}{c}{NQ-Open} \\[-1.3pt]
\midrule
\multirow{2}{*}{10} & Highest & \textbf{68.69}  (-2.52) & \textbf{54.04} (-1.84) & \textbf{56.72} (-2.41) & \textbf{71.58} (-1.80) & \textbf{66.13} (-2.16) &  \textbf{48.93}  (-2.80)    \\
& Lowest & 66.98  (-2.89) & 49.30 (-2.03) & 53.29 (-2.72) & 71.29 (-2.01) & 65.70 (-2.43) &  46.97  (-3.38)  \\
\midrule
\multirow{2}{*}{20} & Highest & \textbf{64.86}  (-2.45) & \textbf{52.05} (-1.99) &\textbf{52.50} (-2.40) & \textbf{69.00} (-1.83) & \textbf{62.97} (-2.21)  &  \textbf{42.25}  (-2.70) \\
& Lowest & 62.60  (-2.83) & 46.91 (-2.03) & 48.51 (-2.72) & 67.68 (-2.01) & 61.05 (-2.43) & 42.09  (-3.23) \\
\midrule
\multirow{2}{*}{30} & Highest & \textbf{57.70}  (-2.52) & \textbf{50.30} (-1.88) & \textbf{50.00} (-2.60) & 64.36 (-1.84) & \textbf{60.95} (-2.41) & \textbf{39.31}  (-2.56) \\
& Lowest & 53.96  (-2.92) & 45.27 (-2.03) & 46.42 (-2.83) & \textbf{65.12} (-2.03) & 59.55 (-2.65)  & 39.12  (-3.05) \\
\bottomrule
\end{tabular}
}
\caption{Instance-level correlation between $\pmi(\question, \contextDpi)$ and answer accuracy. We compute the average answer accuracy over prompts that yield the highest and lowest $\pmi(\question, \contextDpi)$ as the gold document placed at different positions in the document sequence for each instance. The answer accuracy and the average $\pmi(\question, \contextDpi)$ are reported in the table.}
\label{tab:instance-level-correlation-nqopen}
\end{table*}

\paragraph{Experiments on ELI5.} \label{sec:eli5-instance-level}
Compared to NQ-Open, ELI5 is a more challenging long-form QA dataset where questions are mostly about \textit{how}/\textit{why}/\textit{what}, and the answers are expected to be more comprehensive and cover multiple aspects. \looseness=-1
Due to the lack of gold document annotations on ELI5,
we adopt permutations from the cyclic group $(\perm)$ and random shuffling. 
In random shuffling for $K$ documents, we randomly shuffle the document set $K$ (i.e., same as the number of documents) times and obtain $K$ document sequences for consistency. 
Given multiple prompts for a question, among which only the document orders in the $\textcontext$ are different, we calculate the average performance of the prompts with the highest and lowest $\pmi(\question, \contextDpi)$ for each question in the same fashion as described in \cref{sec:nq-open-instance-level} for NQ-Open.
Results in \cref{tab:instance-level-correlation-eli5} show that LMs achieve higher answer accuracy on the prompts with the highest $\pmi(\question, \contextDpi)$, compared with the prompts with the lowest $\pmi(\question, \contextDpi)$. 
This indicates LMs can better answer questions with higher question likelihood through document shuffling, demonstrating the strong instance-level correlation between $\pmi(\question, \contextDpi)$ with answer accuracy. \looseness=-1

\begin{table*}
\centering
\small
\resizebox{0.95\linewidth}{!}{%
\begin{tabular}{cc@{~} ccccc @{}}
\toprule
\#Doc & $\pmi(\question, \contextDpi)$ & \modelname{Mistral-7B-Inst} & \modelname{LLaMA-3-8B} & \modelname{LLaMA-3.1-8B} & \modelname{LLaMA-3-8B-Inst} & \modelname{LLaMA-3.1-8B-Inst}  \\
\midrule
\\[-12.5pt]
\rowcolor{black!10} \multicolumn{7}{c}{ELI5 with Rotational permutation} \\[-1.3pt]
\midrule
\multirow{2}{*}{5} & Highest & \textbf{13.97} (-3.72)  & \textbf{11.37} (-2.23) & \textbf{12.60} (-2.28) & \textbf{14.23} (-2.21) & \textbf{13.97} (-2.26)  \\
& Lowest & 13.50 (-4.06) & 11.10 (-2.39) & 12.50 (-2.43) & 13.17 (-2.54) & 13.93 (-2.48) \\
\midrule
\multirow{2}{*}{10} & Highest & \textbf{15.23} (-3.53) & 11.27 (-2.19) & 12.50 (-2.29) & \textbf{14.50} (-2.10) & \textbf{16.17} (-2.23) \\
& Lowest & 14.47 (-3.99) & \textbf{11.50} (-2.39) & \textbf{13.10} (-2.48) & 14.07 (-2.55) & 15.77 (-2.54) \\
\midrule
\multirow{2}{*}{20} & Highest & \textbf{16.20} (-2.13) & 11.13 (-2.19) & \textbf{12.77} (-2.28) & \textbf{16.20} (-2.13) & \textbf{17.17} (-2.18) \\
& Lowest & 15.80 (-2.73) & \textbf{11.20} (-2.42) & 12.13 (-2.48) & 15.80 (-2.73) & 15.67 (-2.54) \\ 
\midrule
\\[-12.5pt]
\rowcolor{black!10} \multicolumn{7}{c}{ELI5 with Random Shuffling} \\[-1.3pt]
\midrule
\multirow{2}{*}{5} & Highest & \textbf{14.27} (-3.73) & 10.73 (-2.24) & \textbf{12.57} (-2.28) & \textbf{14.10} (-2.23) & \textbf{14.20} (-2.27)  \\  
& Lowest & 14.10 (-4.04) & \textbf{11.20} (-2.39) & 12.33 (-2.42) & 12.77 (-2.52) & 14.00 (-2.48)  \\  
\midrule
\multirow{2}{*}{10} & Highest & \textbf{15.63} (-3.54) & \textbf{11.47} (-2.19) & \textbf{12.73} (-2.29) & \textbf{15.70} (-2.11) & \textbf{16.90} (-2.23) \\  
& Lowest & 15.07 (-3.97) & 11.23 (-2.39) & 12.20 (-2.48) & 14.57 (-2.52) & 16.70 (-2.53) \\  
\midrule
\multirow{2}{*}{20} & Highest & 16.10 (-3.44) & 10.83 (-2.19) & \textbf{12.60} (-2.28) & \textbf{16.13}(-2.14) & \textbf{17.20} (-2.18) \\  
& Lowest & \textbf{16.53} (-4.00) & \textbf{11.20} (-2.42) & 11.87 (-2.49) & 15.53 (-2.71) & 17.10 (-2.54) \\
\bottomrule
\end{tabular}
}
\caption{Instance-level correlation between $\pmi(\question, \context)$ and answer accuracy on ELI5. 
The average QA accuracy is computed over prompts that yield the highest and lowest $\pmi(\question, \context)$ as the input documents are reordered with (1) rotational reordering and (2) random shuffling as introduced in \cref{sec:eli5-instance-level}. The QA accuracy and the average $\pmi(\question, \context)$ are reported in the table.}
\label{tab:instance-level-correlation-eli5}
\vspace{-5pt}
\end{table*}

\section{Improving RAG via Reordering}
In \cref{sec:instance-correlation}, we offered evidence for \Cref{hyp:hypothesis}, i.e., that $\pmi(\question, \contextDpi)$ correlates with model performance.
In light of this finding, we propose two methods to permute the documents presented to the LM in RAG \emph{without} knowledge of the {\answercolor answer}.\looseness=-1

\subsection{Method 1: Search by PMI}
Our empirical findings showed that the permutation of the documents in the {\contextcolor context} that leads to the highest value of $\pmi(\question, \contextDpi)$ leads to superior performance on QA tasks. 
This suggests a natural algorithm
\begin{equation}
 \perm^\star = \argmax_{\perm \in \mathbb{S}_K} \pmi(\question, \contextDpi)
\end{equation}
However, as discussed in \Cref{sec:sets-of-permutation}, the set of all permutations (the symmetric group) $\mathbb{S}_K$ is too large to enumerate.
Thus, we fall back on a simple approximation.
Given a user-provided permutation $\perm$, we search over the cyclic group generated by $\perm$, denoted as $(\perm)$.
Using the notation introduced in \Cref{sec:sets-of-permutation}, we choose
$\permtilde_{k^\star}$ where we select
\begin{equation}
k^\star = \argmax_{k=1}^K \pmi(\question, \context_{\docset}(\permtilde_{k})),
\end{equation} 
where $\permtilde_{k}$ is defined in \Cref{sec:sets-of-permutation}. 

\subsection{Method 2: Search by Curvature}

We now develop a second algorithm based on the observation in \cref{fig:20docs-llama-3.1-all} that accuracy and PMI change simultaneously and exhibit a U-shaped curve as the gold document position within the permutation of documents in $\context$.
Our algorithm is based on a discrete notion of convexity and an assumption based on our findings in \cref{sec:results}, which we introduce in the abstract below.

\paragraph{Technical Interlude Discrete Convexity.}

\def\tauperm{{\tau}}
\def\taupermtilde{{\widetilde{\tauperm}}}
A sequence of real values $\{a_n\}_{n=1}^N$ is called \defn{convex} if we have
\begin{equation}\label{eq:convex}
 \Delta_n^2 \defequals 2 a_n - a_{n+1} - a_{n-1}  \leq 0 
\end{equation}
for all $n \in \{2, \ldots, N-1\}$.
In the abstract, the problem we wish to solve 
is this: Given an arbitrary finite sequence of reals $\{b_n\}_{n=1}^N$, find a permutation $\tau \colon [N] \rightarrow [N]$ that renders $\{b_n\}_{n=1}^N$ convex, i.e., that \Cref{eq:convex} holds after applying the permutation to the sequence's indices. 
We call such a choice of $\tau$ a \defn{convex permutation}.
Note that convex permutations may not always exist.\footnote{E.g., the sequence $[0,1,1,1]$ has no convex permutation.} 
To achieve a U-shape curve, do not just want a convex permutation, but in addition the one that results in a convex sequence that has as much upwards curvature as possible.
In other words, if $\tau$ is a convex permutation, 
then \emph{in addition} we want the following sum to be \emph{minimized}
\begin{subequations}\label{eq:convex-2}
\begin{align}
 \sum_{n=2}^{N-1} \Delta_n^2 &= -(b_{1} + b_{N}) + \sum_{n=2}^{N-1} b_{m} \\
   &= -2(b_{1} + b_{N}) + B \\
   &\leq 0,
\end{align}
\end{subequations}
where $B \defequals \sum_{n=1}^N b_n$.
However, because $B$ is constant, the total curvature induced by a convex permutation $\tau$ \emph{only} depends on $b_{\tau(1)}$ and $b_{\tau(N)}$.
This implies that we simply need to choose the endpoints to be those elements of $\{b_{\tau(n)}\}_{n=1}^N$ that are largest; we can always permute the remaining $(N-2)$ elements to 
ensure the permutation is convex afterward. 
Thus, relaxing the requirement that the permutation be convex, we choose a permutation $\tau$ such that $b_{\tau(1)} + b_{\tau(N)}$ is maximized.
This definition motivates a new definition: We call a sequence $\{b_n\}_{n=1}^N$ is \defn{U-shaped} iff $b_1 \geq b_i$ and $b_N \geq b_i$ for $i \in \{2,3,\cdots,N-1\}$.

\begin{figure}
    \centering    \includegraphics[trim={0mm 60mm 0mm 70mm},clip,width=0.99\linewidth]{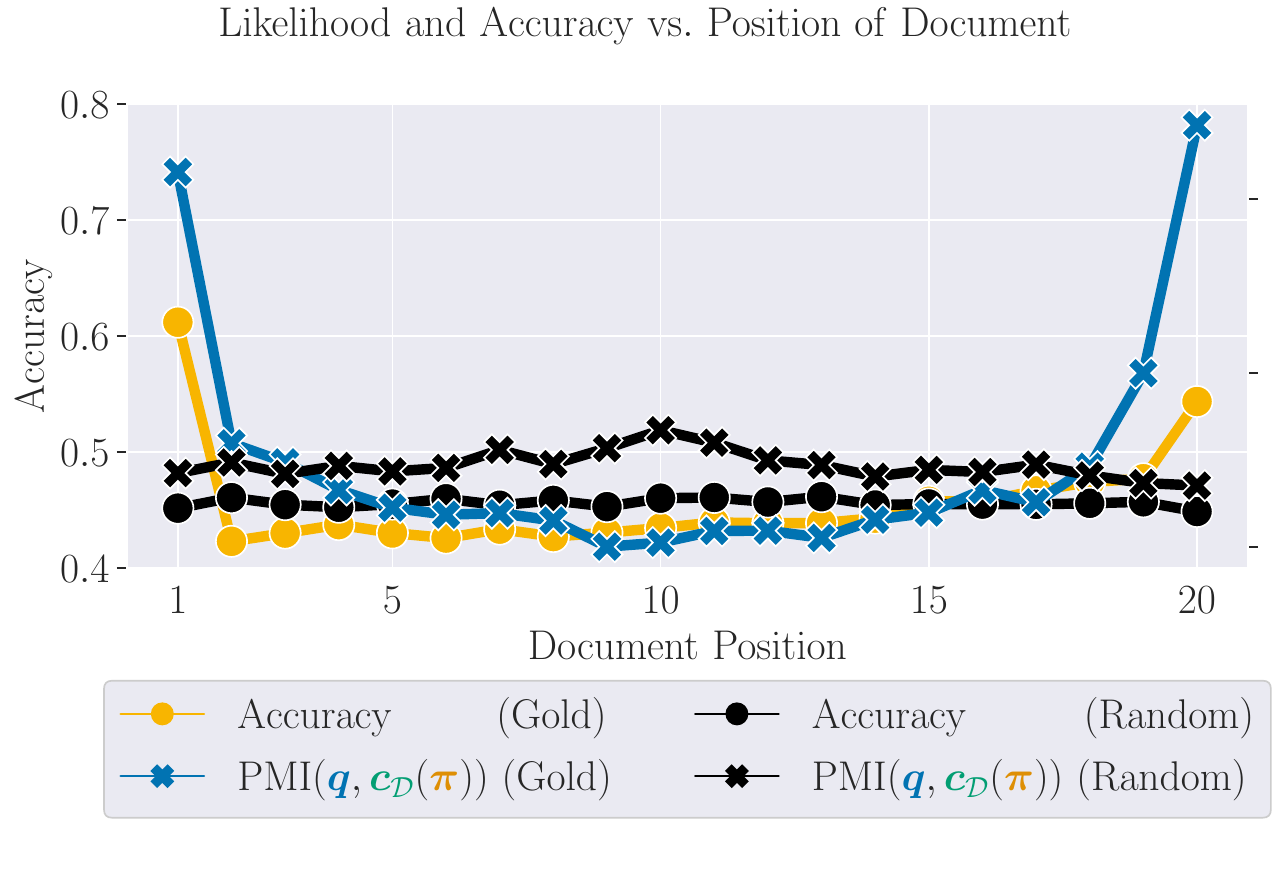}
    \caption{When the position of the gold document changes, both $\pmi(\question, \context)$ and accuracy curves are U-shaped. In contrast, both curves are flat for non-gold (denoted by random) documents.}
    \label{fig:gold-vs-rand}
    \vspace{0pt}
\end{figure}

\paragraph{A Simple Algorithm.}
The abstract discussion in the previous paragraph suggests a simple algorithm. 
First, we construct a real-valued sequence
\begin{equation}
    b_{\tauperm(k)} \defequals \pmi(\question, \context_{\docset}(\permtilde_{k}))
\end{equation}
of length $K$ where $\tauperm \colon [K] \rightarrow [K]$ is a permutation
and $\pmi(\question, \context_{\docset}(\permtilde_{k}))$ is defined in \Cref{sec:sets-of-permutation}.
Then, relaxing the requirement that the permutation be convex, we optimize
\begin{equation}\label{eq:true-method-two}
\tauperm^\star = \argmax_{\tauperm \in (\pi)} b_{\tauperm(1)} + b_{\tauperm(K)}
\end{equation}
While, in general, $\tauperm^\star$ may not be convex, we do have a guarantee that $\tauperm^\star$ will induce a U-shaped sequence.
To compute the optimization problem given in \Cref{eq:true-method-two}, we sort $\tauperm \in (\perm)$ according to $b_{\tauperm(1)} + b_{\tauperm(K)}$ in descending order, obtaining the sequence $\{\tauperm_k\}_{k=1}^{K}$. %
Then, we construct the resulting permutation $\tauperm' = (\tauperm_1(1), \tauperm_2(1), \cdots, \tauperm_K(1))$, among which $\documentd_{\tau(1)}$ is most likely to be the gold document. %

\subsection{Results and Analysis} \label{sec:results-analysis}
Shown in \cref{tab:likelihood-tuning},
both search by PMI and search by curvature can boost answer accuracy.
On NQ-Open, where only one document in the sequence is relevant to the question, gold document reordering significantly improves the answer accuracy and narrows the gap to the upper bound.
Furthermore, on the more challenging and practical QA benchmark ELI5, we also observe a modest improvement in answer accuracy, indicating that improving question likelihoods via document reordering can effectively obtain better LM responses. \looseness=-1

Regarding efficiency, our proposed methods are mildly time-dependent thanks to the parallelizable computation of question likelihoods, where only the LM encoding module is used, with no reliance on LM decoding.\footnote{LM decoding (i.e., generation) requires a runtime approximately proportional to the number of generated tokens. Empirically, the only extra computational time for our methods is on the encoding phase for calculating likelihoods, so the overall runtime is the vanilla ${\rm Runtime}_{\rm LM}$ plus \textit{one} extra LM going through, which is equivalent to generating the response with \textit{one} additional token.}
Shown in \cref{tab:time_efficiency}, in our experiments, the average runtime for decoding a response of an instance in ELI5 is 10 seconds, while it only takes an extra 0.8 seconds and 2 seconds, respectively, to encode the input prompts of naïve likelihood-based selection and gold document reordering. The increment in timely cost is marginal compared with heuristic prompt engineering which requires whole decoding to judge the prompt quality (e.g. an extra 10 seconds for decoding another candidate prompt). \looseness=-1

In summary, both proposed methods are effective and efficient.
Although the improvement on ELI5 is relatively marginal compared to that on NQ-Open, given the more challenging nature of long answers and no specified gold document on ELI5, it still indicates that optimizing prompts with $\lmprefix(\question \mid \context)$ as a gauge is a promising direction. \looseness=-1

\paragraph{Experimental Setup.}
We experiment on the ELI5 dataset and a subset of 500 questions from NQ-Open, using \modelname{Mistral-7B-Inst-v0.3}, \modelname{LLaMA-3.1-8B}, and \modelname{LLaMA-3.1-8B-Inst}. 
Each question is associated with 10 (on ELI5) and 20 (NQ-Open) retrieved documents.

\begin{table}[t!]
\centering \small
\resizebox{\columnwidth}{!}{%
\begin{tabular}{lccc|c} 
\toprule
\multirow{2}{*}{Model} & \multirow{2}{*}{Baseline} & \multirow{2}{*}{PMI} & \multirow{2}{*}{Curvature} & Upper \\
 & &  &  & Bound \\
\midrule \\[-11.5pt]
\rowcolor{black!10} \multicolumn{5}{c}{NQ-Open (Answer Accuracy)} \\[-1.3pt]
\midrule
{\modelname{Mistral}} & 62.89   & 65.18  & \textbf{65.72} & 69.24 \\
{\modelname{LLaMA-3.1}}  & 47.74 & 51.29 & \textbf{51.36} & 66.88 \\ 
{\modelname{LLaMA-3.1-Inst}}  & 61.49  & 63.34  & \textbf{63.56} & 66.35  \\
\midrule \\[-11.5pt]
\rowcolor{black!10} \multicolumn{5}{c}{ELI5 (Answer Accuracy)} \\[-1.3pt]
\midrule
{\modelname{Mistral}} & 15.35 & \textbf{15.63} & 15.40 & - \\
{\modelname{LLaMA-3.1}} & 12.61 & 12.73 & \textbf{13.33} & - \\
{\modelname{LLaMA-3.1-Inst}} & 16.14 & \textbf{16.90} & 16.83 & - \\
\bottomrule
\end{tabular}
}
\caption{Performance of our methods on NQ-Open and ELI5, the number of documents $K$ is set to 20 and 10, respectively.
\modelname{Mistral}, \modelname{LLaMA} and \modelname{LLaMA-Inst} stands for \modelname{Mistral-7B-Inst-v0.3}, \modelname{LLaMA-3.1-8B} and \modelname{LLaMA-3.1-8B-Inst} respectively. 
Baseline refers to the mean performance over $K$ random document shuffling on each instance. The upper bound on NQ-Open is calculated as the performance when positioning the gold document at the beginning of the document sequence, which is not applicable for ELI5 since no gold document is marked in this practical dataset. \looseness=-1} \label{tab:likelihood-tuning}
\vspace{-5pt}
\end{table}

\subsection{Synthetic Experiment}
Real-world datasets might have been used during the training of LLMs.
Thus, their likelihoods might exhibit an \emph{exposure bias} \cite{bengio-2015,ranzato-2016,cotterell2024formalaspectslanguagemodeling}.
To avoid such potential bias, we follow \citet{liu-etal-2024-lost} and conduct a synthetic key--value retrieval experiment.

\begin{table}
\centering \small
\begin{tabular}{c|cc}
\toprule
Decoding & Likelihood Based  & Gold Document \\
\midrule
10s & 0.8s & 2s \\ 
\bottomrule
\end{tabular}
\caption{The average runtime for decoding an LLM response v.s. the extra time for the two proposed methods.}
\label{tab:time_efficiency}
\vspace{-5pt}
\end{table}

\paragraph{Key--Value Retrieval.}
To imitate question-answering tasks on random strings, we construct Python-style key--value pairs in which the keys and values are UUID strings of 32 hexadecimal digits.
An example is given in \cref{fig:kv-retrieval-example}.
In \crefrange{tab:llama-3.1-8b-inst-kv}{tab:mistral-7b-inst-kv}, we observe that both $\pmi(\question, \context)$ and $\lmprefix( \answer \mid \context \bcdot \question)$ show synchronous U-shaped patterns as the location of the key in context changes, consistent with the RAG-based QA experiments in \cref{sec:results}, indicating the generalizability of the findings on unseen data.\looseness=-1

\section{Discussion} 
\label{sec:discussion}
\subsection{Instruction-tuned vs. Base Models}
In our analysis, we find base LMs, e.g., \modelname{LLaMA-3-8B}, tend to be more sensitive to the permutation of the documents.
Specifically, we observe that QA performance drops when the gold document is placed in the middle of the document sequence.
On the other hand, the performance of instruction-tuned models is more robust to permutations of the documents in the context, as shown in \cref{fig:20docs-llama-3.1-all}. 
However, we still do observe the existence U-shaped curve, but the drop in QA performance is less significant for the instruction-tuned model when the gold document is positioned at the middle.\looseness=-1

The fact that PMI serves as a useful gauge for both the base and instruction-tuned models suggests that $\pmi(\question, \contextDpi)$ is affected little by the instruction tuning.

\subsection{When Context is placed after Question}
In our experiments, we only explore the correlation between PMI and accuracy when the question follows the context. 
However, one could also use a prompt template in which the context follows the question.
We remark that in this case, PMI can be computed according to the equation 
\begin{equation}
    \pmi(\question, \context) = \log \frac{\lmprefix(\context \mid \question)}{\lmprefix(\context)}.   
\end{equation}

\section{Conclusion}
In our study, we analyzed the relationship between the PMI between question and context $\pmi(\question, \contextDpi)$ and question-answering performance under the retrieval-augmented generation framework. 
Through experimentation, we demonstrated that $\pmi(\question, \contextDpi)$ is affected by the order of documents in the input context.
We find evidence for a positive correlation between question likelihood and answer accuracy at both the corpus level and instance level.
Our findings show that it is possible to use $\pmi(\question, \contextDpi)$ to gauge language model performance and improve the quality of input prompts. 
We propose two practical methods for prompt optimization based on these findings. 
Experimental results show that both effectively and efficiently improve LM's accuracy on QA tasks, demonstrating that using PMI as a gauge for optimizing prompts is a promising direction.

\section*{Limitations}
One major limitation of our work is that only open-source LMs are studied in this work since we need full access probabilities under the LM.
Thus, closed language models such as \modelname{GPT-4} cannot be used for selecting permutations 

Besides, our prompt modification is limited to document permutation in this work.
Other prompt modifications may also contribute to obtaining a higher $\pmi(\question, \context)$ and improve QA performance. 
Considering that in this work we are taking the first step towards exploring the feasibility of prompt optimization without LM decoding, proving our hypothesis, and managing to optimize prompts with our findings, we leave other prompt optimizations for future study.

\bibliography{custom,anthology}

\clearpage

\onecolumn

\appendix

\section{Related Work}

\subsection{Prompt Engineering}
Prompt engineering is important for making the best use of LMs in real-world applications \citep{giray2023prompt, ekin2023prompt,gonen-etal-2023-demystifying}.
The most straightforward prompt engineering method is to manually design prompts using heuristics, which requires human experts to design prompts based on domain-specific knowledge and select the prompts that lead to better performance on downstream tasks \citep{zhou2023large, marvin2023prompt}.
Meanwhile, another line of work explores automatic approaches for prompts engineering \citep{gao-etal-2021-making, pryzant-etal-2023-automatic}.
However, they both require decoding for outputs from LMs to evaluate the quality of prompts, thus incurring high computational costs. \looseness=-1

\subsection{Retrieval-Augmented Generation}
Retrieval-augmented generation is a  technique for improving LMs' ability to solve knowledge-intensive tasks \citep{rag-lewis-2020, NEURIPS2021_3df07fda, borgeaud2022improving}.
In the RAG framework, a set of documents relevant to a user query is retrieved from an external source and inserted into prompts as a $\textcontext$, to provide additional information to the LM and improve response quality \citep{petroni-etal-2020-how,rag-lewis-2020}. 
RAG tasks can be divided into two types: short-form and long-form, depending on the topic of the questions and the format of the expected answers. 
Short-form QA \citep{izacard-grave-2021-leveraging, liu-etal-2024-lost} usually concerns factual questions about real-world facts. 
The expected answers are often unambiguous and concrete words or short phrases. 
Long-form QA \citep{fan-etal-2019-eli5, gao-etal-2023-enabling} involves \textit{how}, \textit{why}, and \textit{what} questions that seek more comprehensive responses. \looseness=-1

\subsection{Effect of Document Order}

\citet{liu-etal-2024-lost} finds that LMs perform better when the document with relevant information is positioned at the beginning or the end of the prompt using under RAG framework.\footnote{In \citeposs{liu-etal-2024-lost} experimental settings, the gold document is unique in a prompt for each question.} 
Specifically, when moving the task-relevant information from the beginning to the end of the document sequence, answer accuracy exhibits a U-shaped trend on a multi-document QA task and a synthetic key--value retrieval task, both using RAG pipelines. 
However, \citet{liu-etal-2024-lost} mainly focuses on an empirical study with less in-depth analysis, resulting in a gap between the phenomenon and its practical implications. 
In this work, we attempt to bridge this gap.\looseness=-1

\section{Illustration of Evaluation Metrics}
\label{app:illustration_evaluation_metric}
The evaluation metrics for NQ-Open and ELI5 are illustrated with two examples in \cref{fig:evaluation-illustration}.

\begin{figure}
    \centering
    \includegraphics[width=1.0\linewidth]{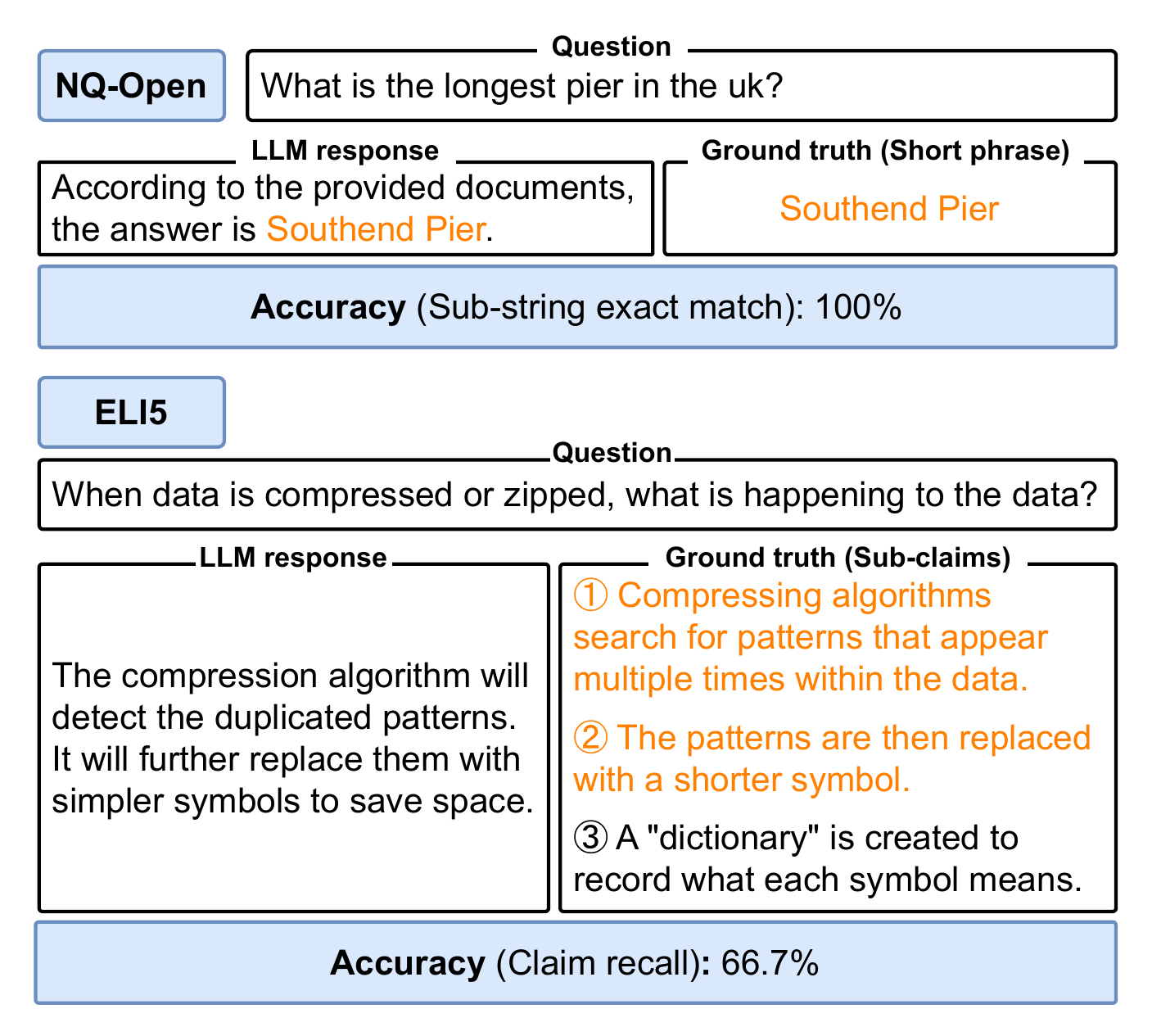}
    \caption{Evaluation metrics used in our experiments. On NQ-Open, the evaluation metric is exact string match. On ELI5, a pretrained NLI model is used to evaluate whether the LM output entails the reference claims.}
    \label{fig:evaluation-illustration}
\end{figure}

\section{Datasets}\label{app:datasets}

\paragraph{NQ-Open.}
We first experiment on the NQ-Open dataset following \citet{liu-etal-2024-lost}. 
This dataset covers 2655 factual questions curated from the Natural Questions dataset \citep{kwiatkowski-etal-2019-natural,lee-etal-2019-latent} under \texttt{CC-BY-SA-3.0} license. 
Each question is accompanied by $K$ documents retrieved from Wikipedia, among which \emph{exactly one} contains the answer to the question, namely the gold document. 
The remaining $k-1$ documents are termed \defn{distractors}, which are relevant to the topic of the question but do not contain any ground truth answers, retrieved using Contriever \cite{izacard2022unsuperviseddenseinformationretrieval}. 
In our experiments, the total number of documents $K$ is taken to be $\{ 10, 20, 30\}$.\footnote{We remark that NQ-Open was specifically synthesized to examine how answer accuracy is affected by changing the position of relevant information.
In real-world applications, the retrieved documents for one question may contain multiple gold documents or none. 
Nevertheless, it mimics the RAG setup underlying many commercial generative search and QA systems. \looseness=-1} \looseness=-1

\paragraph{ELI5.} To validate the generality of our findings, we also experiment on an open-ended non-factual QA dataset ELI5 \citep{fan-etal-2019-eli5} with \texttt{BSD} license. 
ELI5 consists of questions beginning with \textit{how}, \textit{why} or \textit{what} curated from the Reddit forum ``Explain Like I’m Five''\footnote{\url{https://www.reddit.com/r/explainlikeimfive/}}, where the answers are expected to be more comprehensive and diverse. 
Each question is accompanied by $K$ documents retrieved from Sphere \citep{piktus2021web}---a filtered version of Common Crawl\footnote{\url{https://commoncrawl.org}}, 
where $K$ is taken to be $\{5, 10, 20\}$ to avoid truncation due to the long questions and LMs responses for the long-form QA task. 
In contrast to NQ-Open, ELI5 does not provide the annotations of gold documents,
which aligns with real-world RAG application scenarios, making it a more practical and challenging dataset \citep{nakano2021webgpt, menick2022teaching, liu-etal-2023-evaluating}.  \looseness=-1

\section{Prompt Templates} \label{app:templates}

The prompt templates used for our experiments are given in \cref{fig:multidoc-qa-ICQ-nqopen,fig:multidoc-qa-ICQ-nqopen-ELI5,fig:synthetic-kv-retrieval}.

\begin{figure}[!t]
\begin{tcolorbox} \tt \small
Write a high-quality answer for the given question using only the provided search results (some of which might be irrelevant).
\\

{\goldcontextcolor Document [1](Title: Southend Pier) Southend Pier is a major landmark in \textellipsis} \\

{\contextcolor Document [2](Title: Llandudno Pier) Llandudno Pier Llandudno Pier is a Grade II* listed pier in the seaside resort of Llandudno\textellipsis} \\

{\contextcolor Document [3](Title: Garth Pier) Garth Pier Garth Pier is a Grade II listed structure in Bangor\textellipsis} \\

{\hspace*{\fill} \contextcolor \textellipsis \hspace*{\fill}} \\

Question: {\questioncolor what is the longest pier in the uk}
\\

According to the provided documents, the answer is {\answercolor Southend Pier}.
\end{tcolorbox}
\caption{An example prompt and LM output on NQ-Open.
The prompt comprises (1) an instruction that describes the task to be solved, (2) a $\textcontext$ that contains the information for solving the task, in which the {\goldcontextcolor gold document} contains the ground truth answer, and (3) a $\textquestion$ that describes the specific query. At the end of the prompt, we append an exemplar output that gives the ground truth $\textanswer$ to the $\textquestion$ for evaluating the likelihood of the answer. \looseness=-1}\label{fig:multidoc-qa-ICQ-nqopen}
\end{figure} 

\begin{figure}[!t]
\begin{tcolorbox} \tt \small
Instruction: Write an accurate, engaging, and concise answer for the given question using only the provided search results (some of which might be irrelevant). Use an unbiased and journalistic tone.
\\

{\contextcolor Document [1](Title: Trash Islands - the Ocean Garbage Patch): Trash Islands Trash Islands of the Pacific and Atlantic Oceans\textellipsis} \\

{\contextcolor Document [2](Title: Where does our garbage go? - Sea Turtle Camp): Pacific Garbage Patch Landfills are a common human solution for disposing of trash on land\textellipsis} \\

{\contextcolor Document [3](Title: Plastic pollution crisis: How waste ends up in our oceans – Y108): our ecosystems as a whole. Plastic is non-biodegradable. Every year, about 8-million tons of plastic\textellipsis} \\

{\hspace*{\fill} \contextcolor \textellipsis \hspace*{\fill}}
\\

Question: {\questioncolor how does so much of our trash end up in the ocean?}
\\

According to various sources, a significant portion of the world's trash ends up in the ocean due to a combination of factors. While it's often\textellipsis  individuals is necessary to mitigate the problem of plastic pollution in the world. \\
\small{[Answer length: 242 words]}
\end{tcolorbox}
\caption{An example prompt and LM output on ELI5. The prompt comprises (1) an instruction that describes the task to be solved, consistent with previous works on ELI5 \citep{gao-etal-2023-enabling, qi2024model}, (2) a $\textcontext$ that contains the information for solving the task, but \emph{no {\goldcontextcolor gold document}} is marked, and (3) a $\textquestion$ that describes the specific query. At the end of the prompt, we append an exemplar output that gives the ground truth $\textanswer$ to the $\textquestion$ for evaluating the likelihood of the answer. \looseness=-1}\label{fig:multidoc-qa-ICQ-nqopen-ELI5}
\end{figure}

\begin{figure}[!t]
\begin{tcolorbox} \tt \small
\begin{lstlisting}[language=Python,basicstyle=\tiny]
{
    "749d280d-8d74-4a2b-87fa-e2a13b689892": 
        "51f95eb8-1f16-4bbf-a7be-6109e581fc04",
    "6618b34a-08b6-46a8-a438-aedc1a2a4635": 
        "3e93dc61-1e82-46b1-94be-7ef2e63746e5",
    ...
}

Key: "749d280d-8d74-4a2b-87fa-e2a13b689892"
Value:
\end{lstlisting}
\end{tcolorbox}
\caption{An example of synthetic data for key--value retrieval.}\label{fig:synthetic-kv-retrieval}
\end{figure}

\section{Full Results on NQ-Open} \label{app:full-mdqa}

We show the full results on NQ-Open in \Crefrange{fig:10docs-llama-3-all-1}{fig:30docs-llama-3-all-1}.

\begin{figure*}[!ht]
    \centering
    \begin{subfigure}{0.48\textwidth}
        \centering
        \includegraphics[trim={0mm 60mm 0mm 70mm},clip,width=\linewidth]{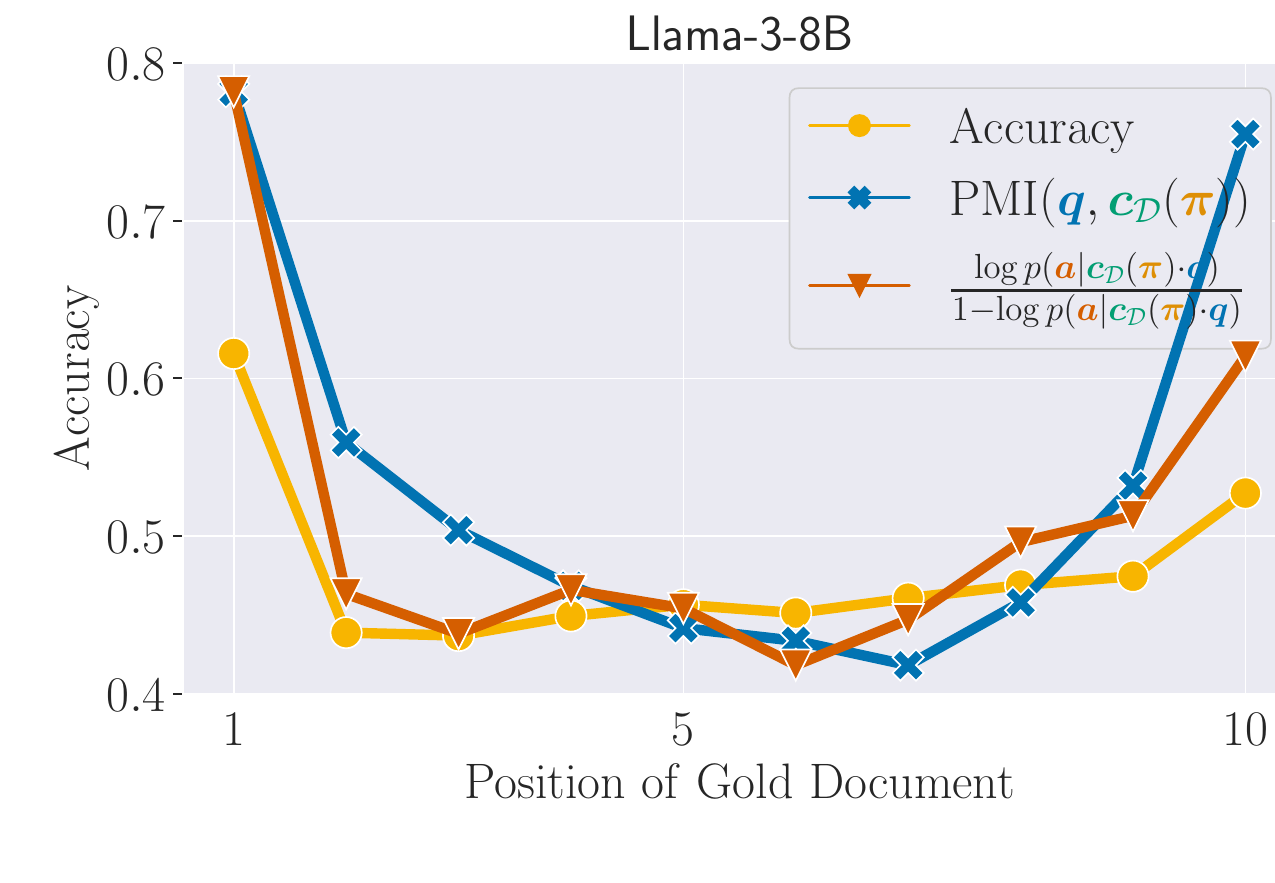}
        \caption{}
    \end{subfigure}
    \begin{subfigure}{0.48\textwidth}
        \centering
        \includegraphics[trim={0mm 60mm 0mm 70mm},clip,width=\linewidth]{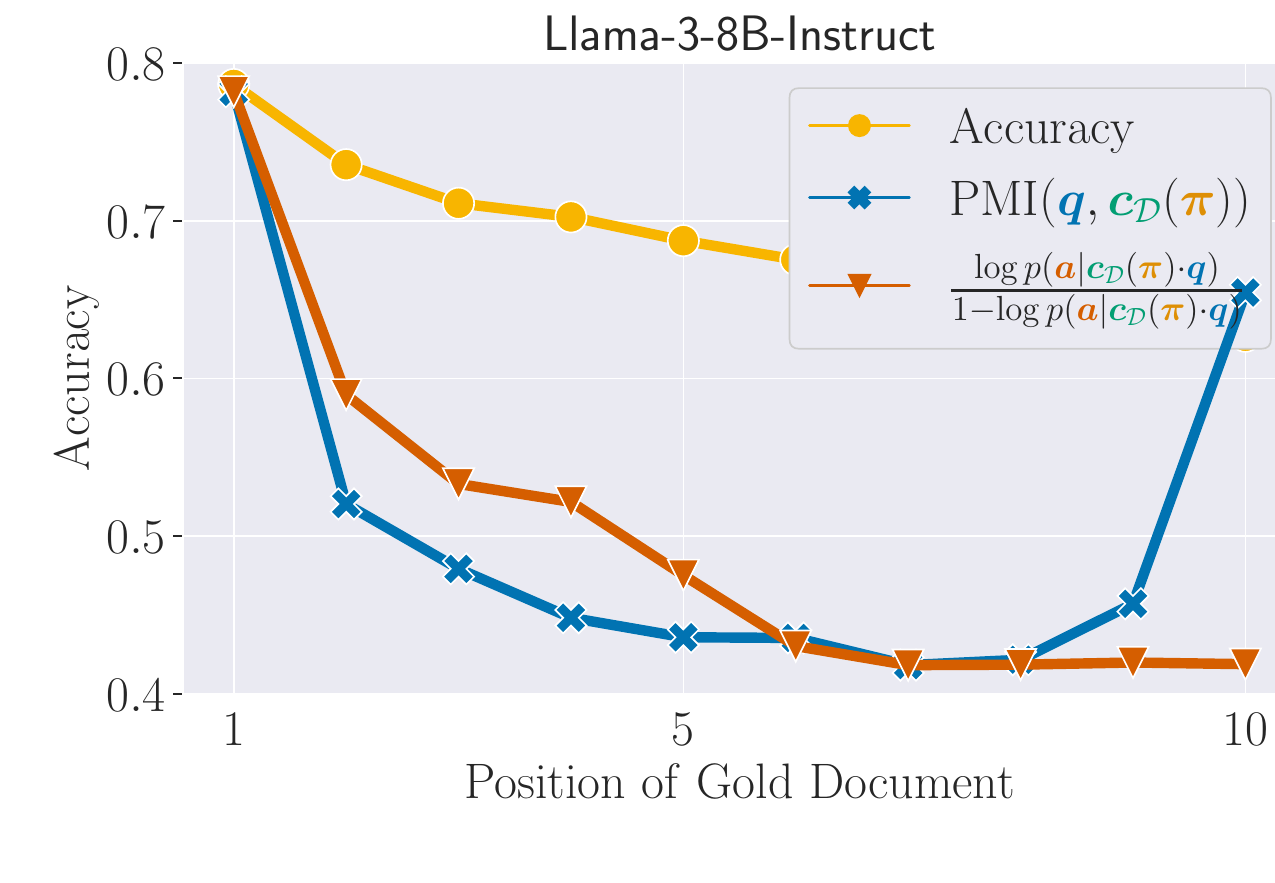}
        \caption{}
    \end{subfigure}
    \begin{subfigure}{0.48\textwidth}
        \centering
        \includegraphics[trim={0mm 60mm 0mm 70mm},clip,width=\linewidth]{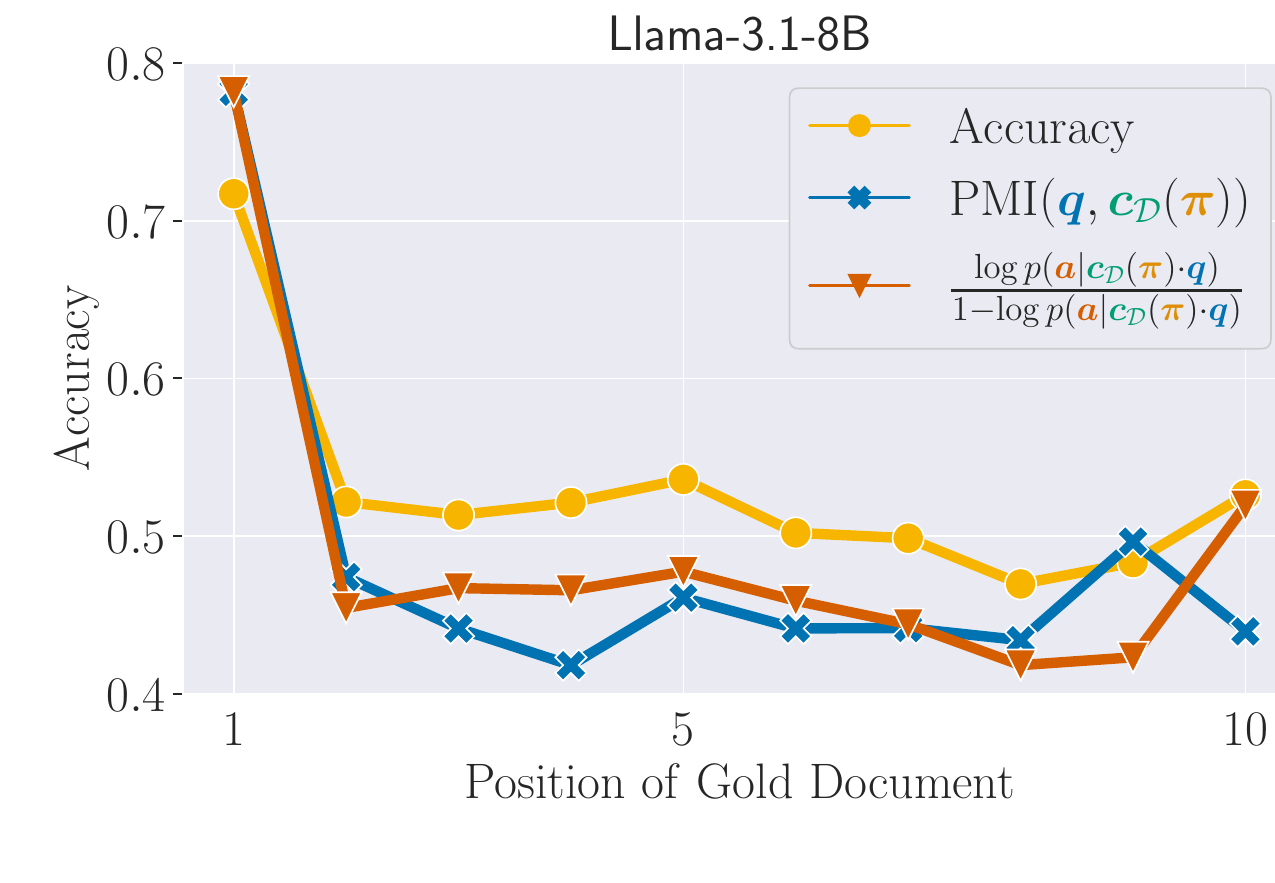}
        \caption{}
    \end{subfigure}
    \begin{subfigure}{0.48\textwidth}
        \centering
        \includegraphics[trim={0mm 60mm 0mm 70mm},clip,width=\linewidth]{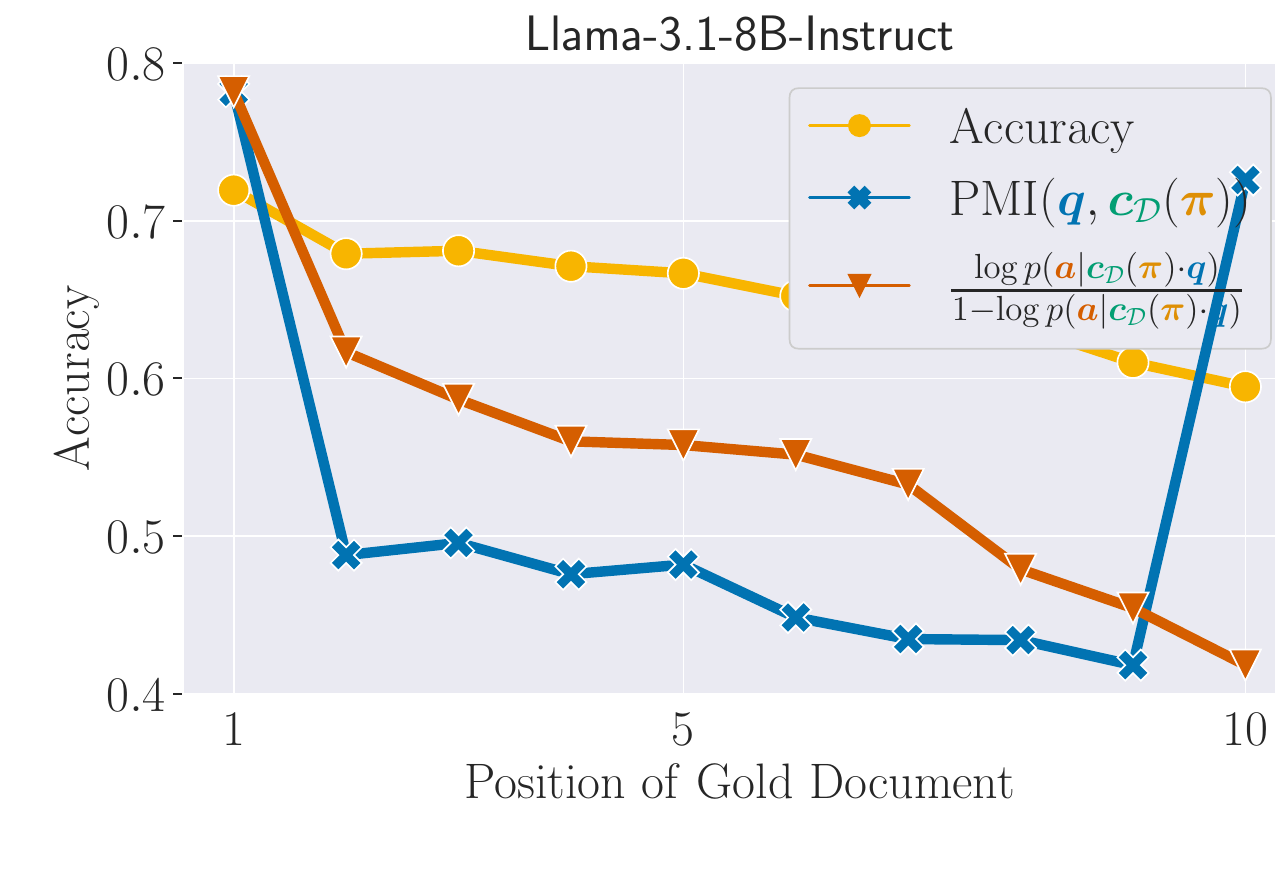}
        \caption{}
    \end{subfigure}
    \begin{subfigure}{0.48\textwidth}
        \centering
        \includegraphics[trim={0mm 60mm 0mm 70mm},clip,width=\linewidth]{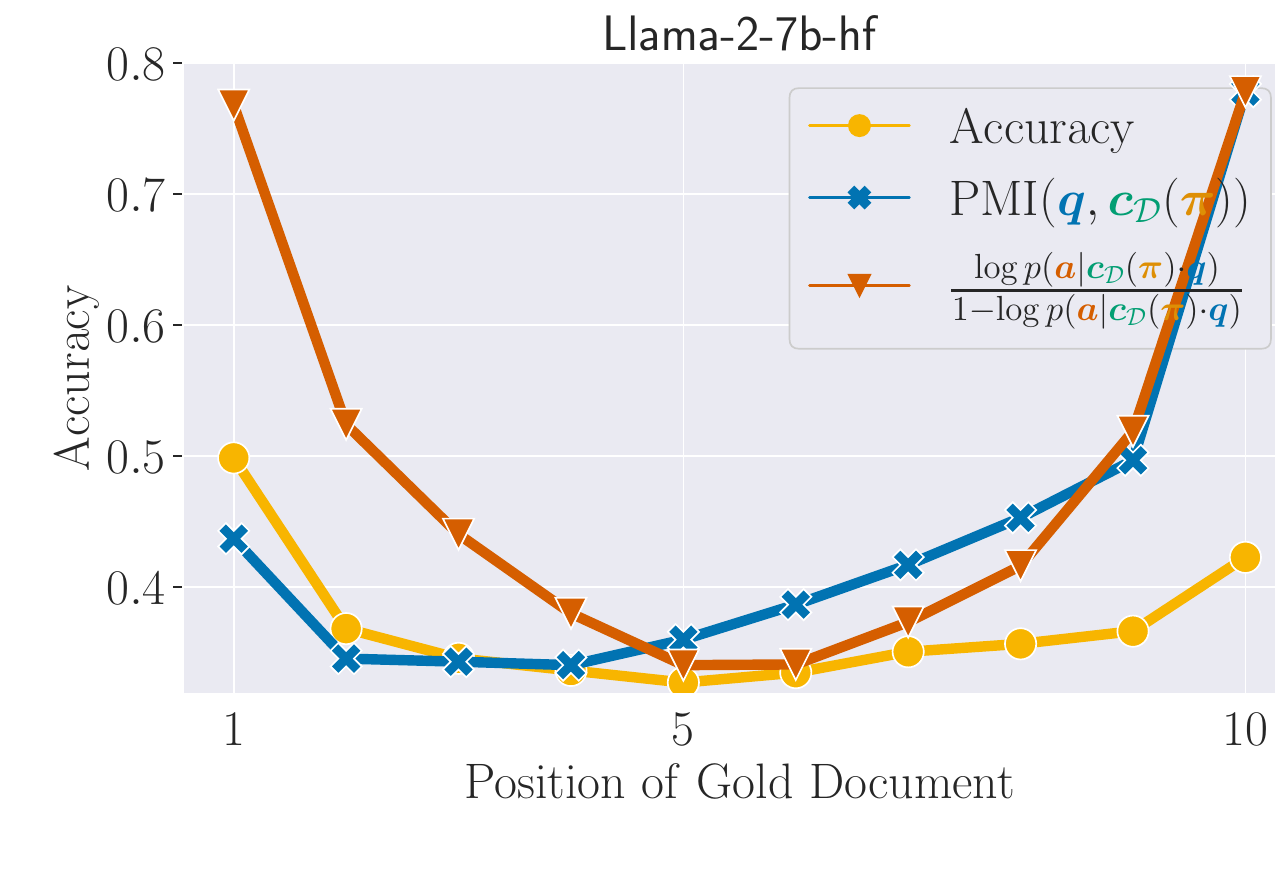}
        \caption{}
    \end{subfigure}
    \begin{subfigure}{0.48\textwidth}
        \centering
        \includegraphics[trim={0mm 60mm 0mm 70mm},clip,width=\linewidth]{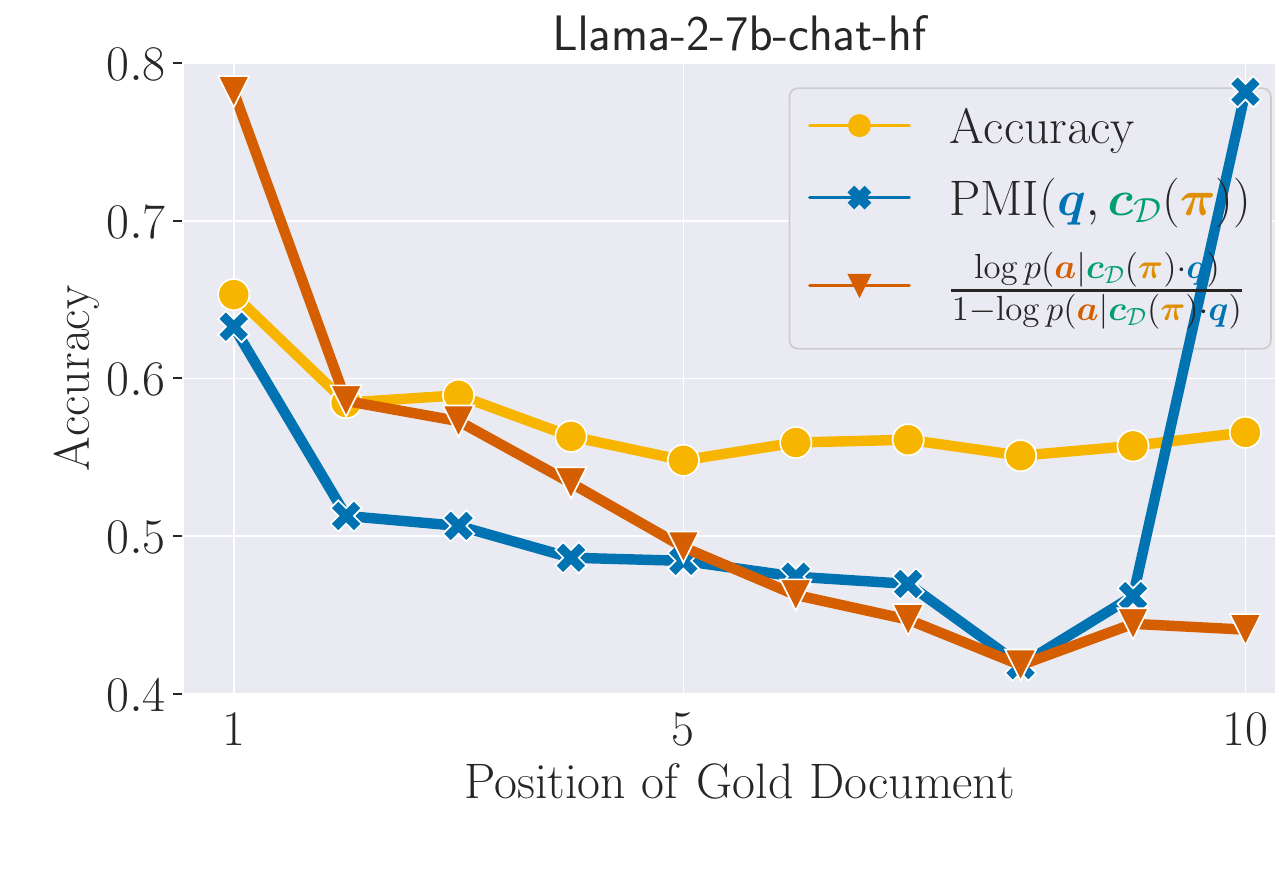}
        \caption{}
    \end{subfigure}
    \begin{subfigure}{0.48\textwidth}
        \centering
        \includegraphics[trim={0mm 60mm 0mm 70mm},clip,width=\linewidth]{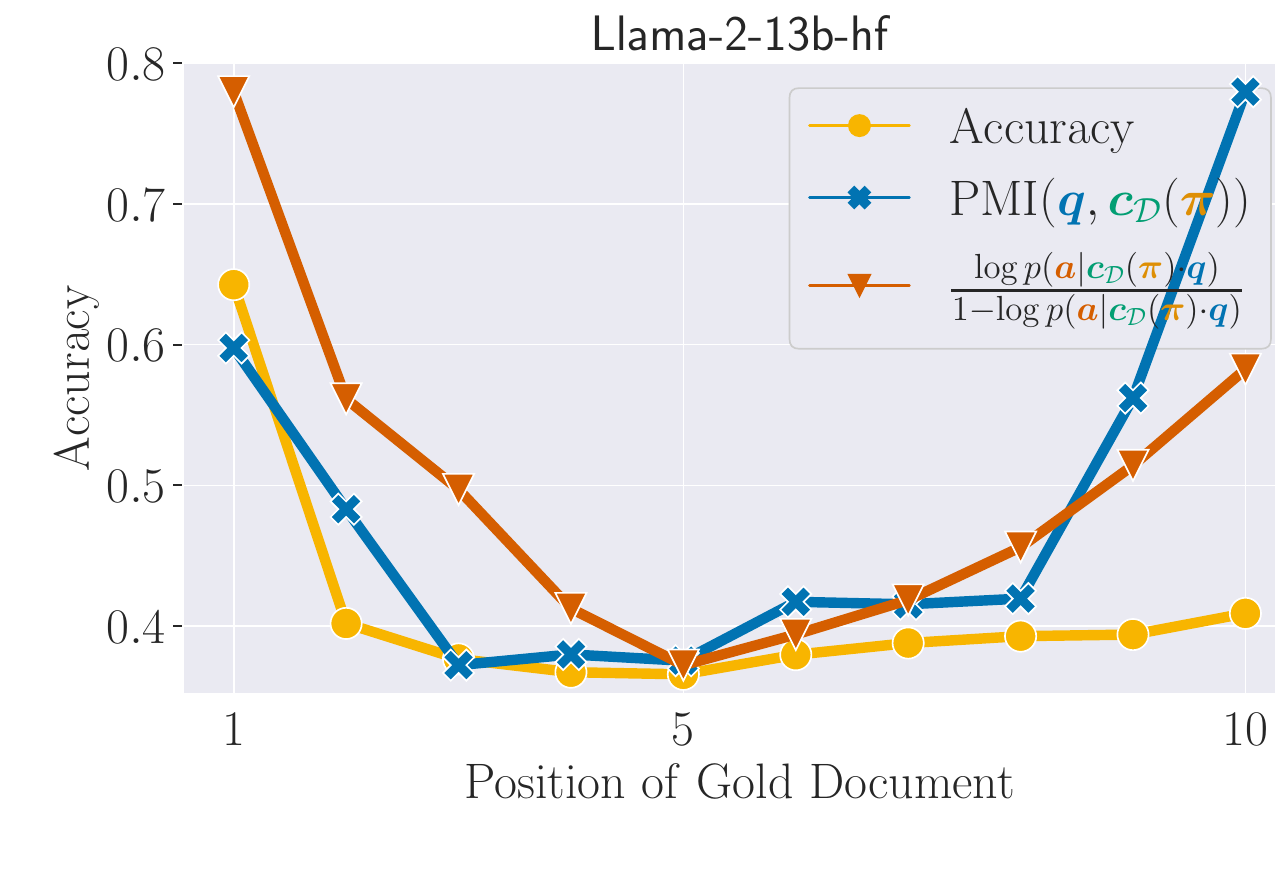}
        \caption{}
    \end{subfigure}
    \begin{subfigure}{0.48\textwidth}
        \centering
        \includegraphics[trim={0mm 60mm 0mm 70mm},clip,width=\linewidth]{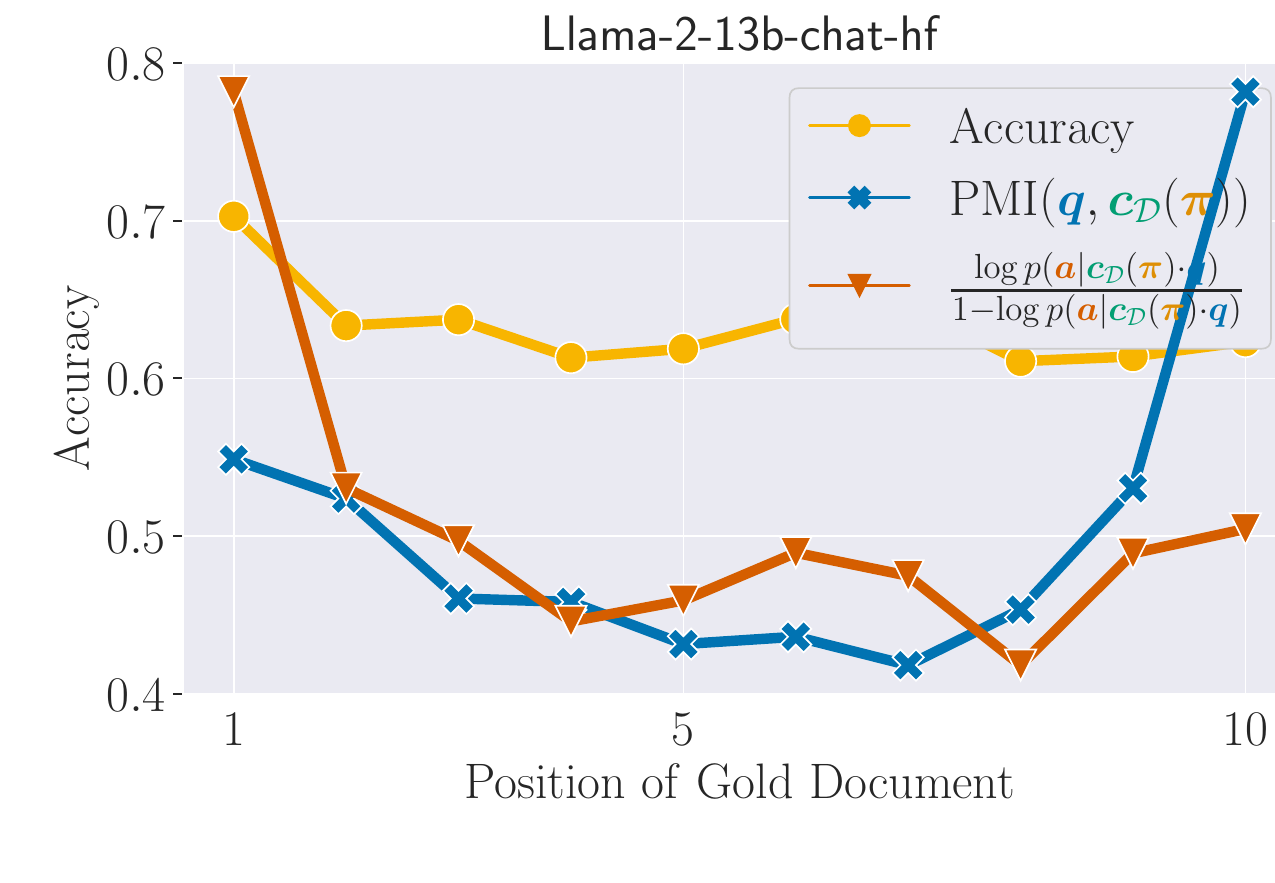}
        \caption{}
    \end{subfigure}
    \caption{QA accuracy, PMI, and log odds ratio of answer likelihood on 10 docs.}
    \label{fig:10docs-llama-3-all-1}
\end{figure*}

\begin{figure*}[!ht]
    \centering
    \begin{subfigure}{0.48\textwidth}
        \centering
        \includegraphics[trim={0mm 60mm 0mm 70mm},clip,width=\linewidth]{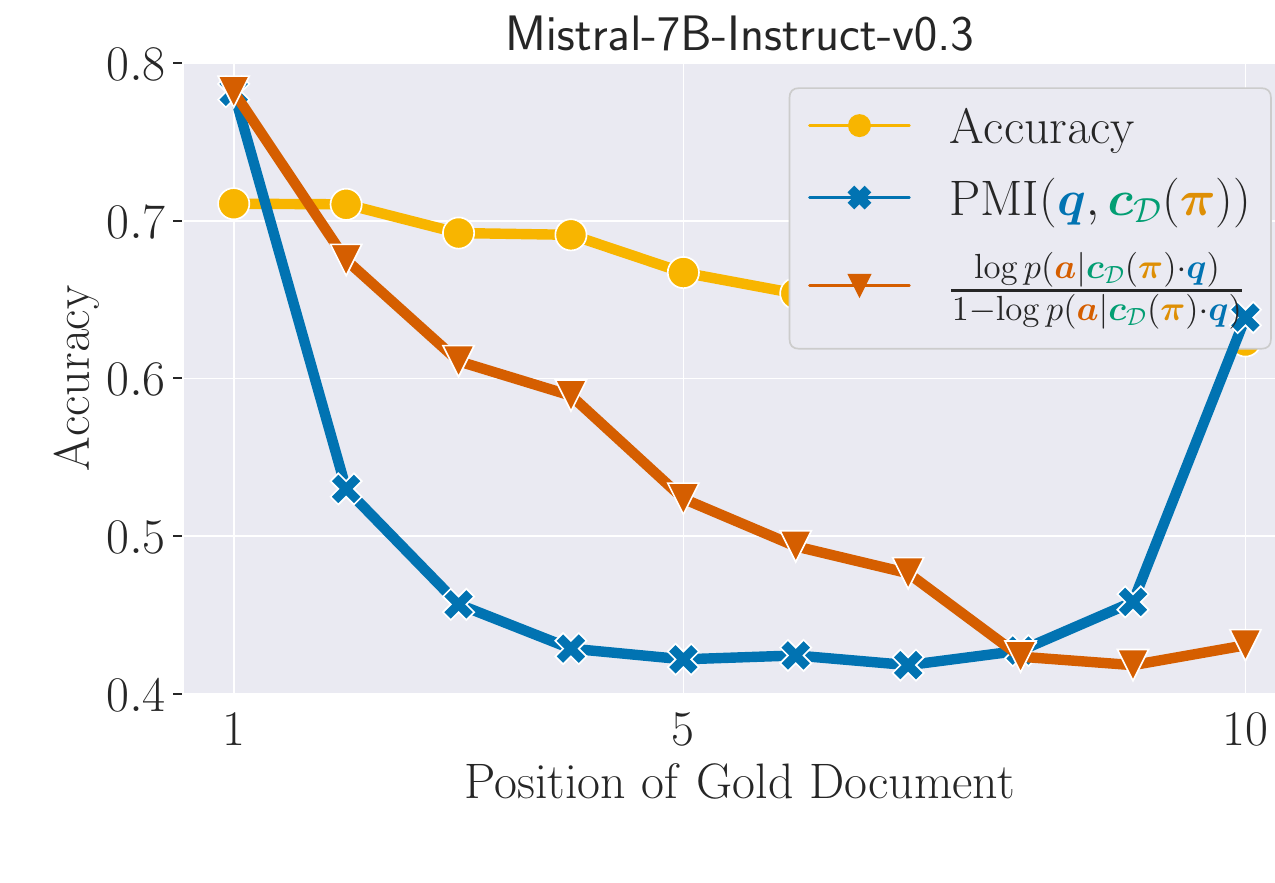}
        \caption{}
    \end{subfigure}
    \begin{subfigure}{0.48\textwidth}
        \centering
        \includegraphics[trim={0mm 60mm 0mm 70mm},clip,width=\linewidth]{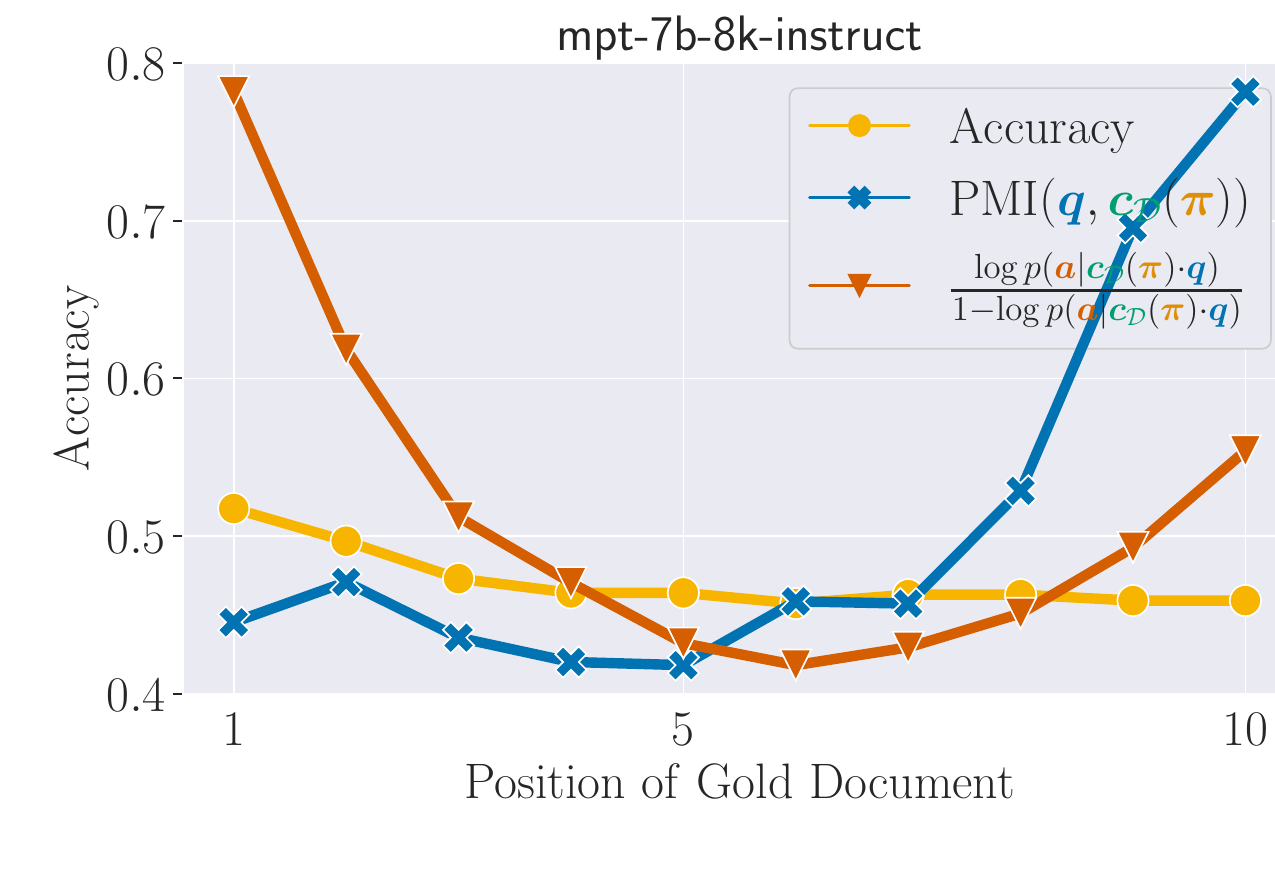}
        \caption{}
    \end{subfigure}
    \caption{QA accuracy, PMI, and log odds ratio of answer likelihood on 10 docs.}
    \label{fig:10docs-llama-3-all-2}
\end{figure*}

\begin{figure*}[!ht]
    \centering
    \begin{subfigure}{0.48\textwidth}
        \centering
        \includegraphics[trim={0mm 60mm 0mm 70mm},clip,width=\linewidth]{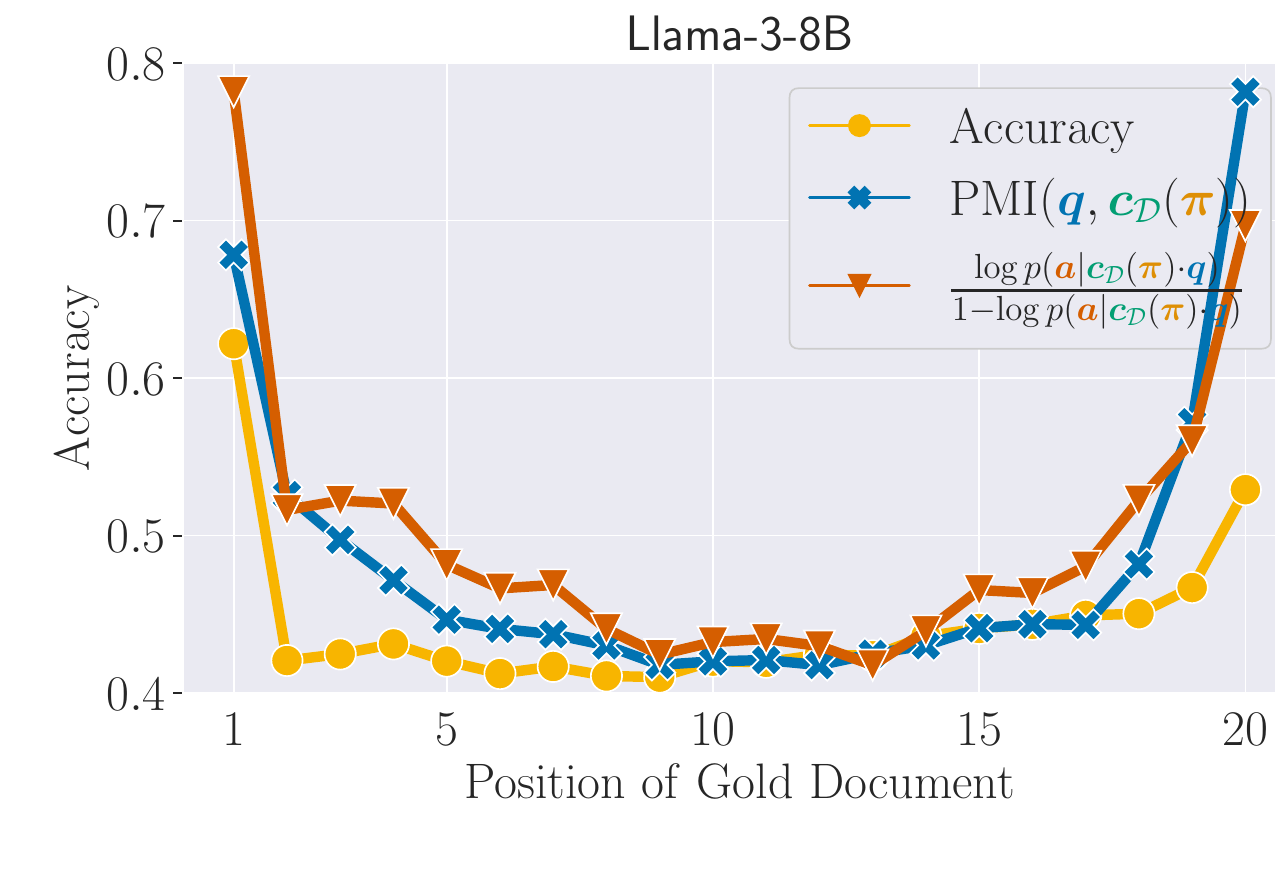}
        \caption{}
    \end{subfigure}
    \begin{subfigure}{0.48\textwidth}
        \centering
        \includegraphics[trim={0mm 60mm 0mm 70mm},clip,width=\linewidth]{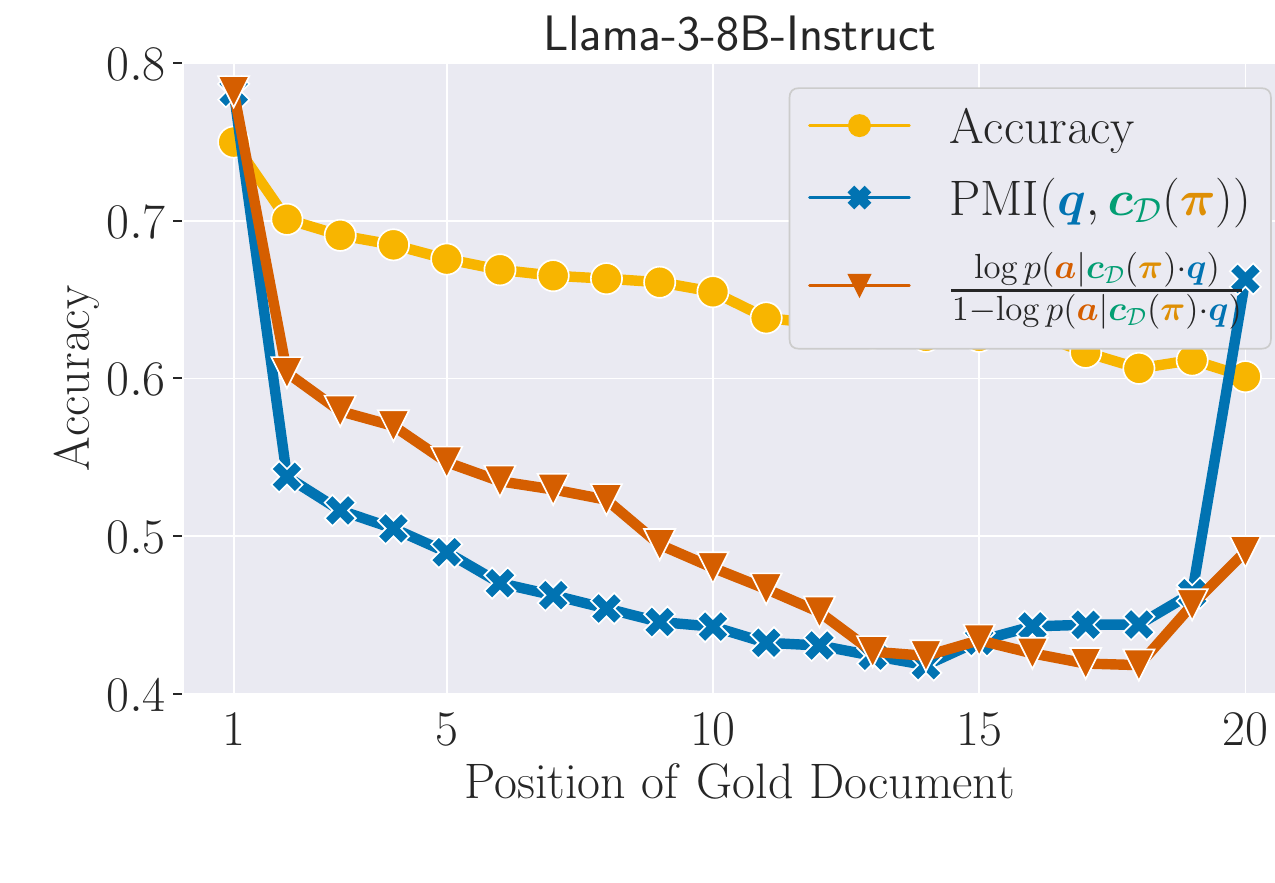}
        \caption{}
    \end{subfigure}
    \begin{subfigure}{0.48\textwidth}
        \centering
        \includegraphics[trim={0mm 60mm 0mm 70mm},clip,width=\linewidth]{plots/all_results_nqopen/one_model_result_20.Llama-3.1-8B.combined..pdf}
        \caption{}
    \end{subfigure}
    \begin{subfigure}{0.48\textwidth}
        \centering
        \includegraphics[trim={0mm 60mm 0mm 70mm},clip,width=\linewidth]{plots/all_results_nqopen/one_model_result_20.Llama-3.1-8B-Instruct.combined..pdf}
        \caption{}
    \end{subfigure}
    \begin{subfigure}{0.48\textwidth}
        \centering
        \includegraphics[trim={0mm 60mm 0mm 70mm},clip,width=\linewidth]{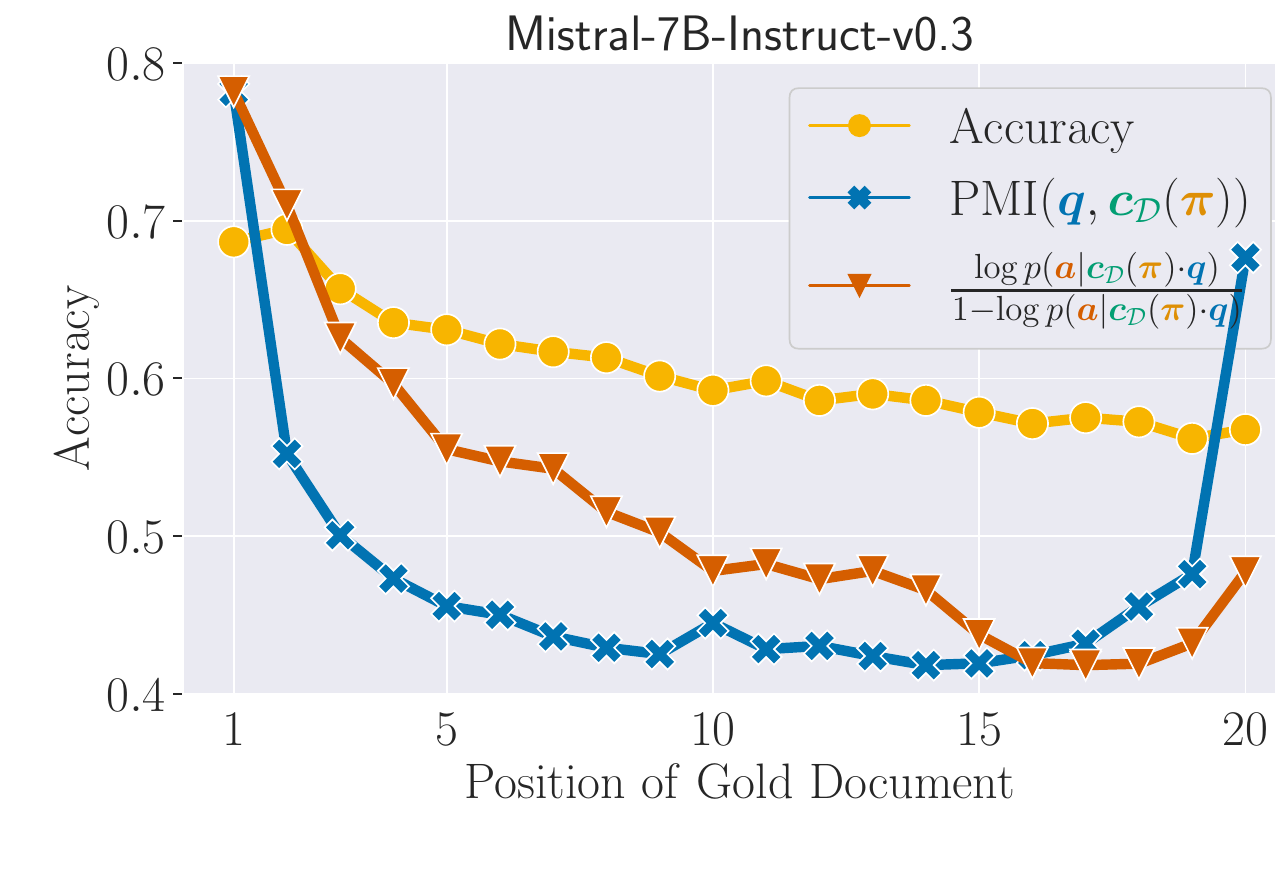}
        \caption{}
    \end{subfigure}
    \begin{subfigure}{0.48\textwidth}
        \centering
        \includegraphics[trim={0mm 60mm 0mm 70mm},clip,width=\linewidth]{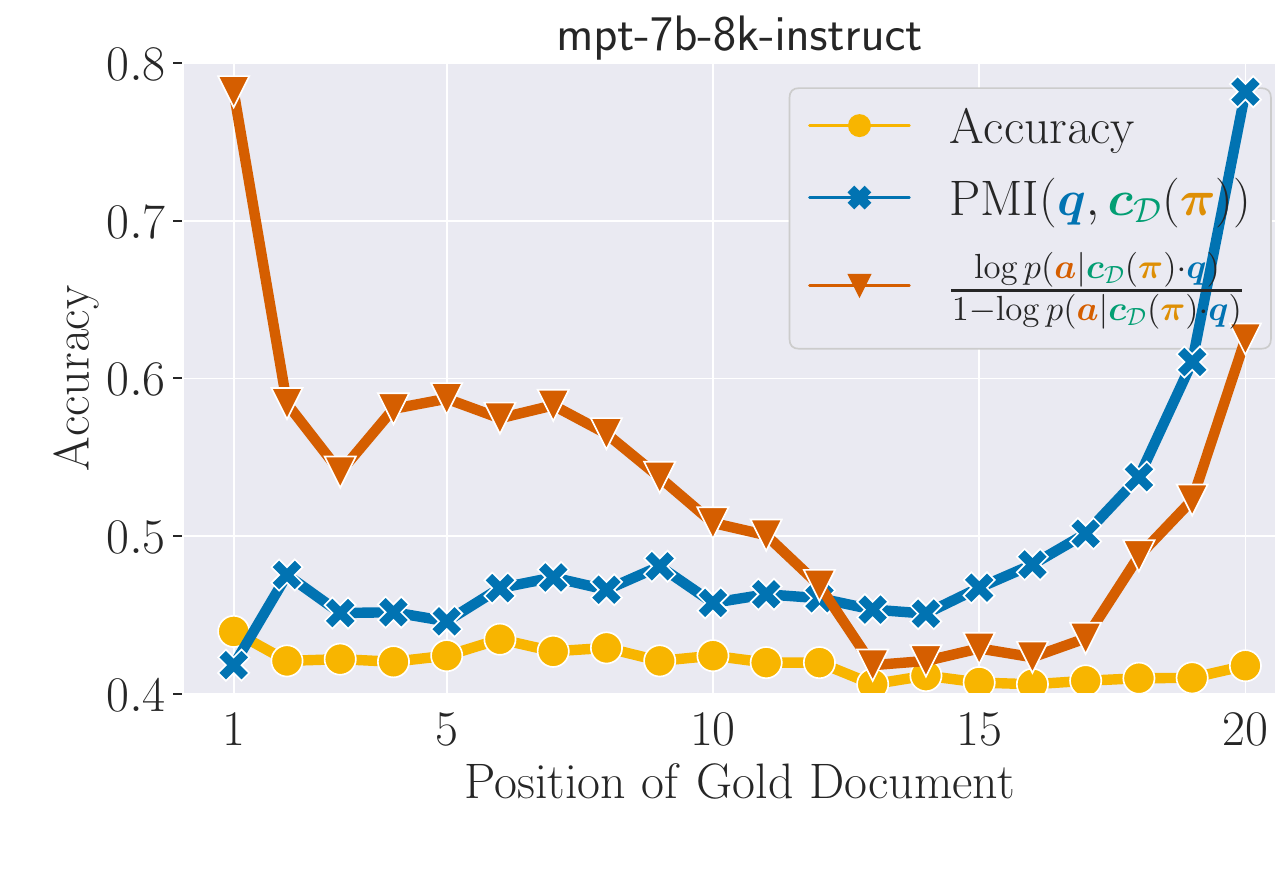}
        \caption{}
    \end{subfigure}
    \caption{QA accuracy, PMI, and log odds ratio of answer likelihood on 20 docs.}
    \label{fig:20docs-llama-3-all-1}
\end{figure*}

\begin{figure*}[!ht]
    \centering
    \begin{subfigure}{0.48\textwidth}
        \centering
        \includegraphics[trim={0mm 60mm 0mm 70mm},clip,width=\linewidth]{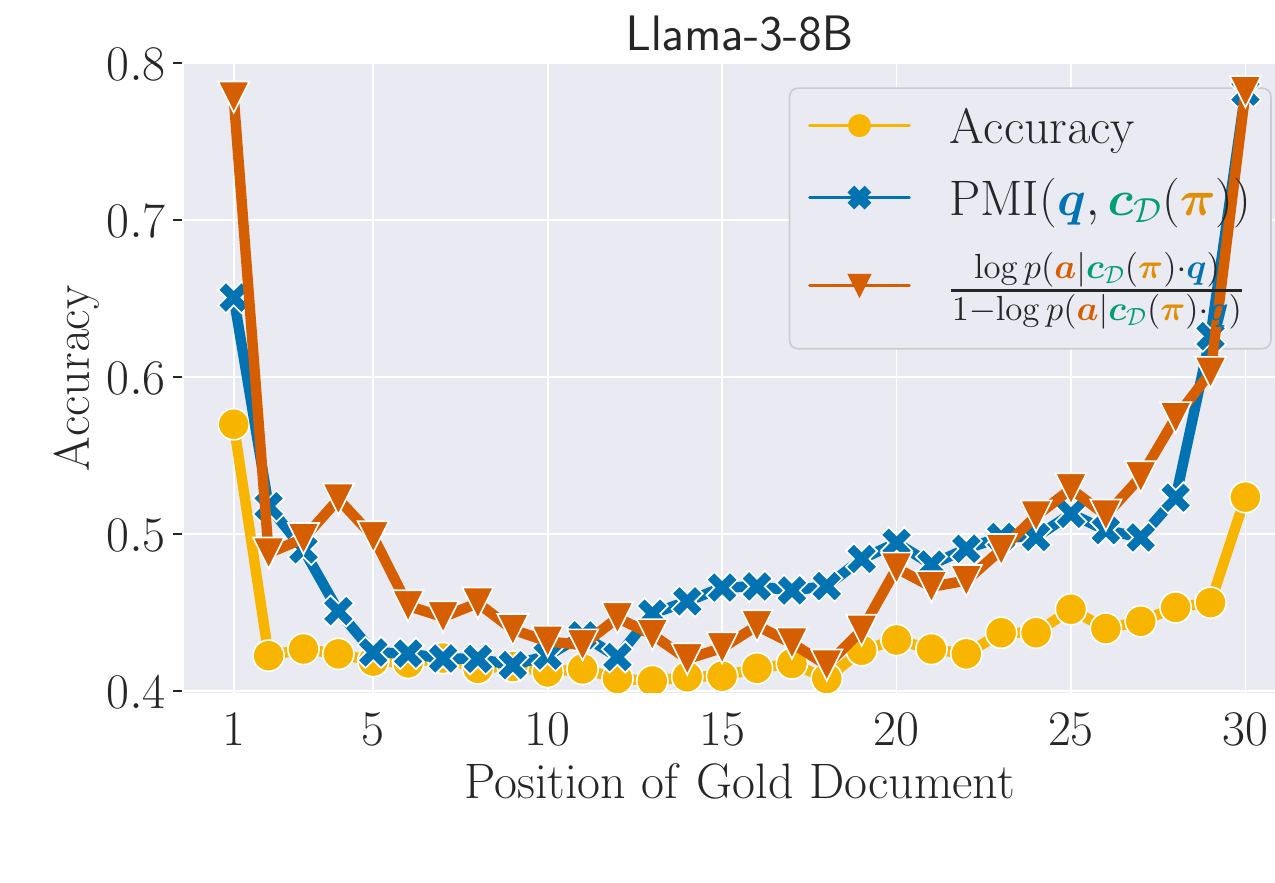}
        \caption{}
    \end{subfigure}
    \begin{subfigure}{0.48\textwidth}
        \centering
        \includegraphics[trim={0mm 60mm 0mm 70mm},clip,width=\linewidth]{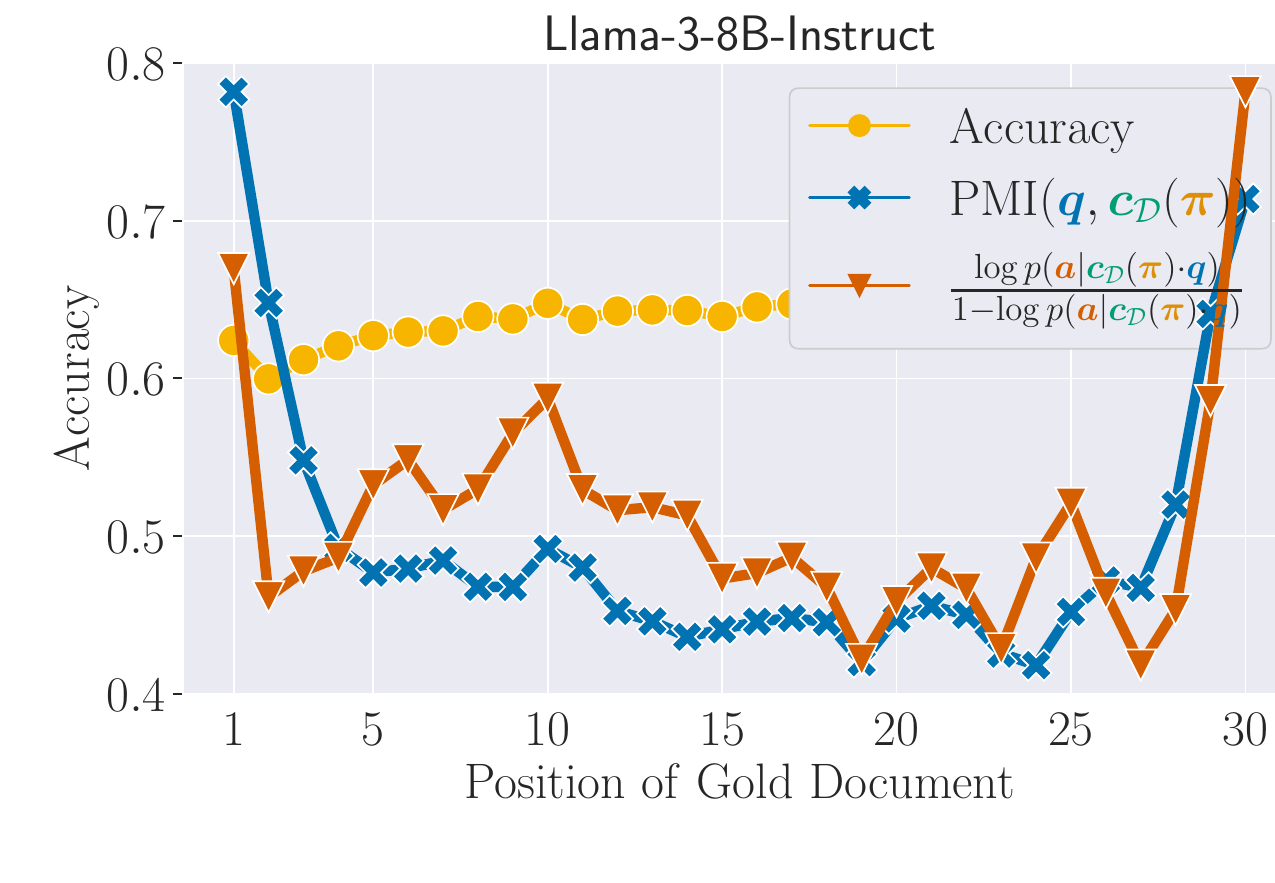}
        \caption{}
    \end{subfigure}
    \begin{subfigure}{0.48\textwidth}
        \centering
        \includegraphics[trim={0mm 60mm 0mm 70mm},clip,width=\linewidth]{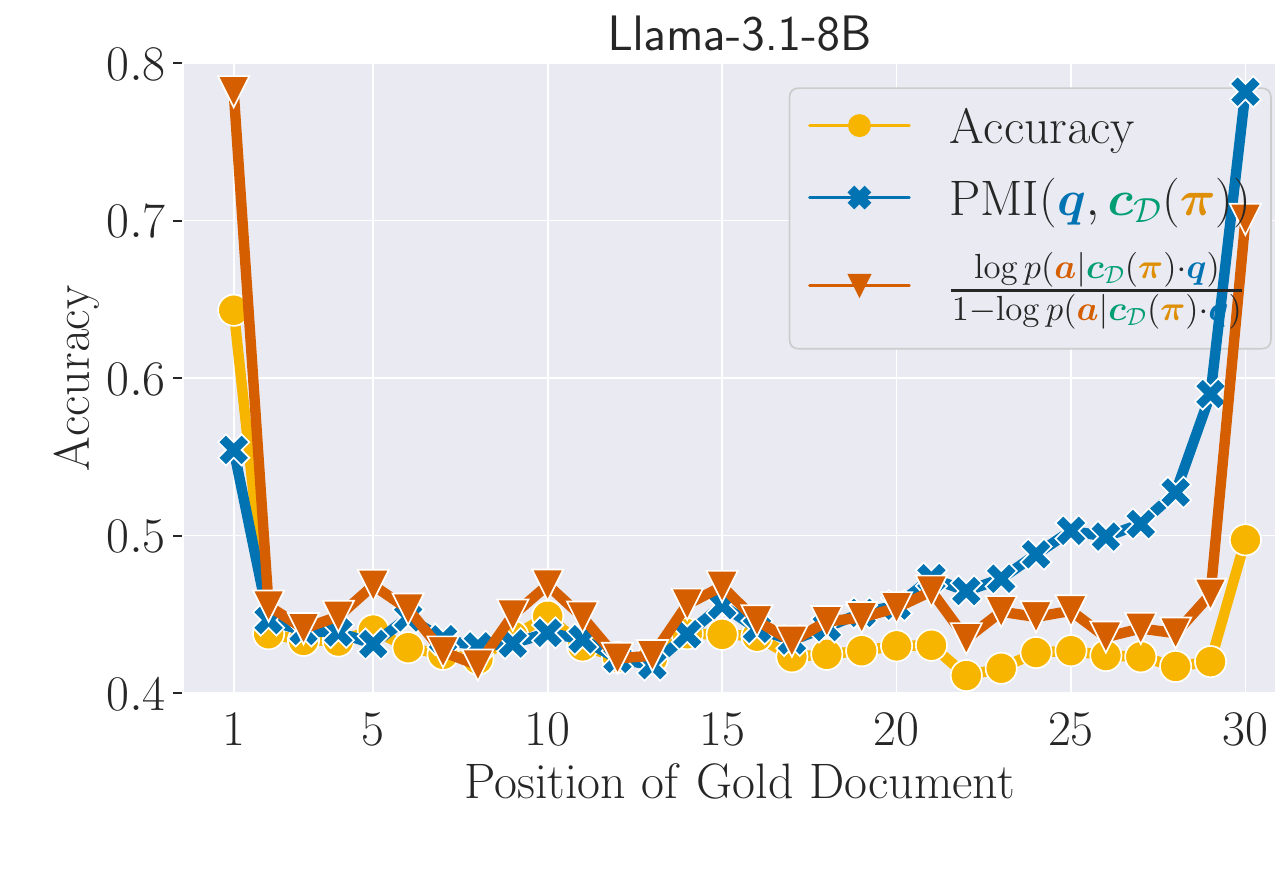}
        \caption{}
    \end{subfigure}
    \begin{subfigure}{0.48\textwidth}
        \centering
        \includegraphics[trim={0mm 60mm 0mm 70mm},clip,width=\linewidth]{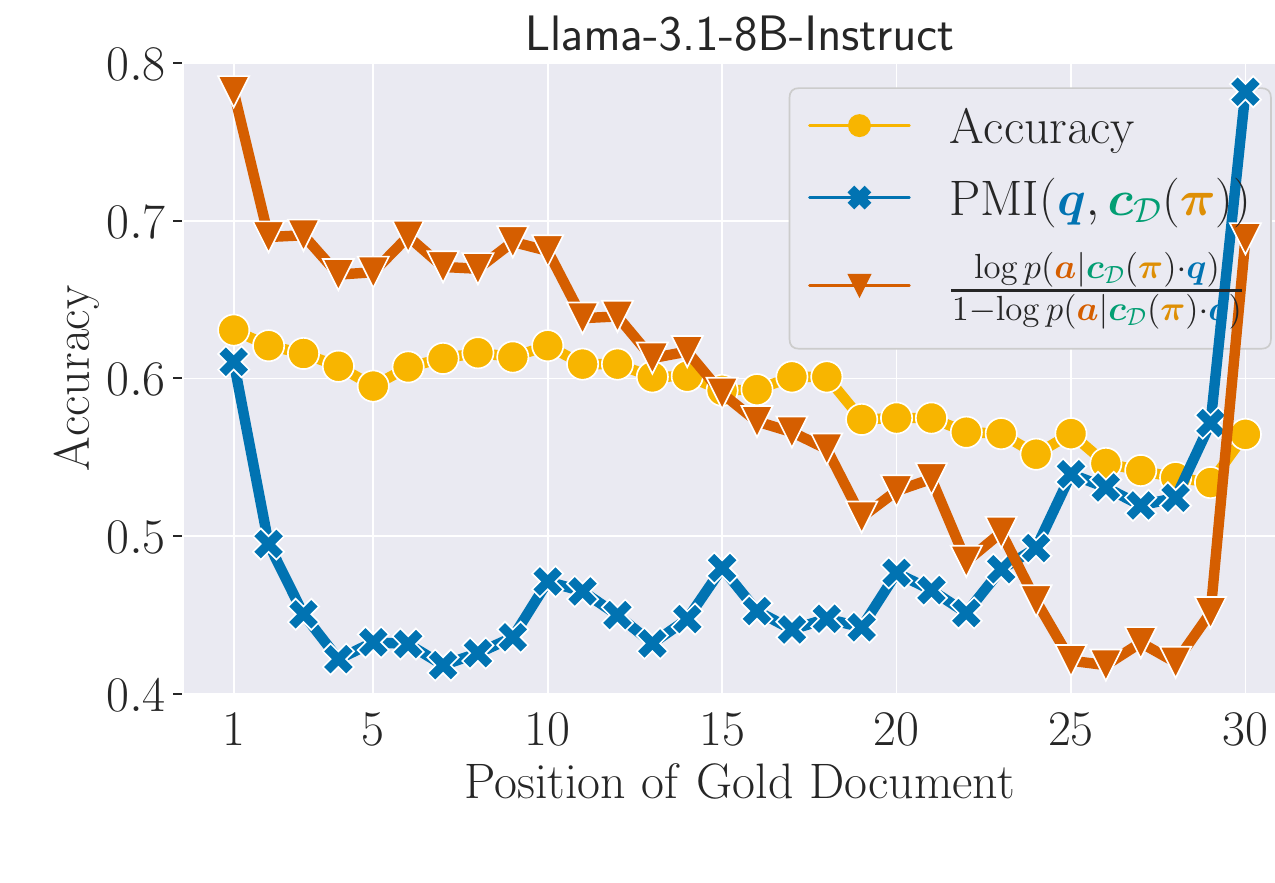}
        \caption{}
    \end{subfigure}
    \begin{subfigure}{0.48\textwidth}
        \centering
        \includegraphics[trim={0mm 60mm 0mm 70mm},clip,width=\linewidth]{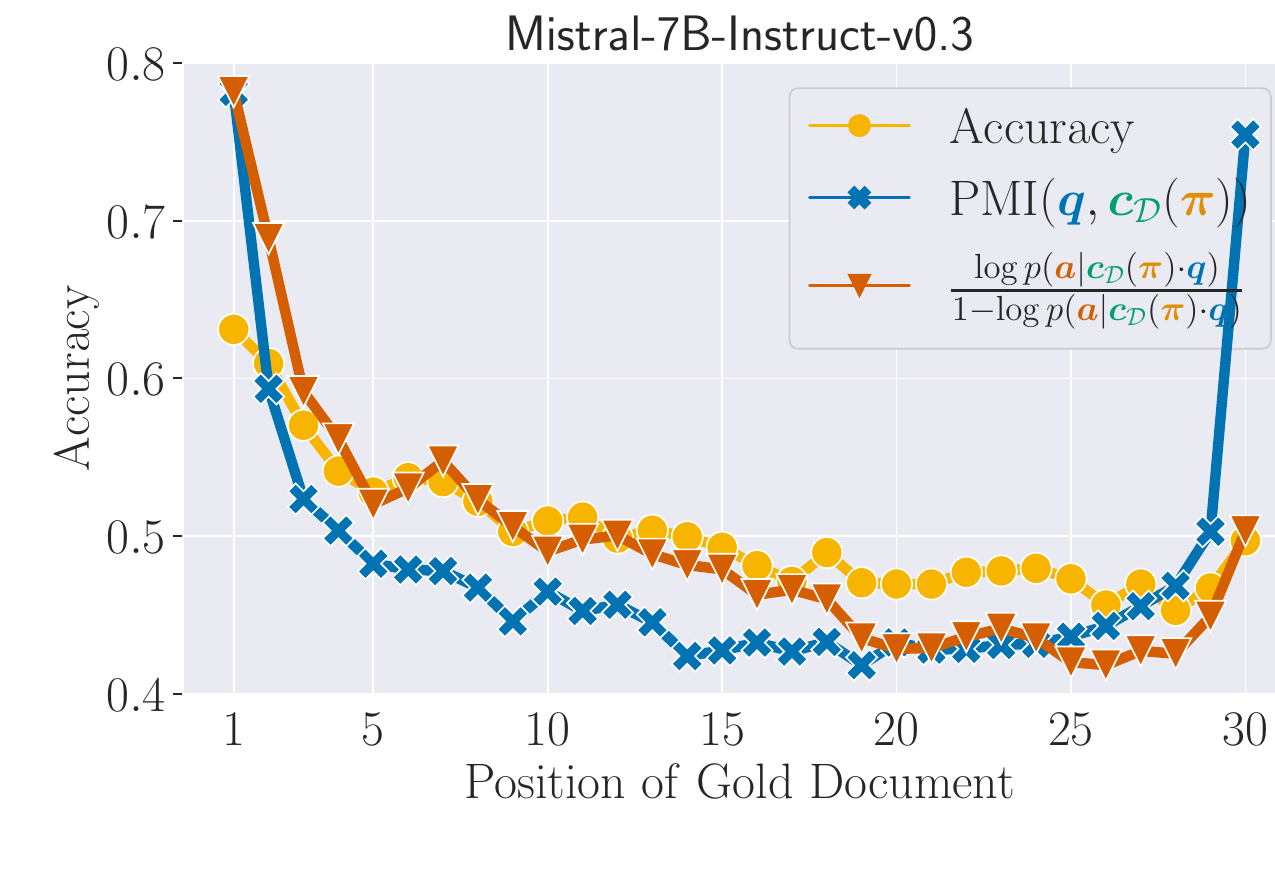}
        \caption{}
    \end{subfigure}
    \begin{subfigure}{0.48\textwidth}
        \centering
        \includegraphics[trim={0mm 60mm 0mm 70mm},clip,width=\linewidth]{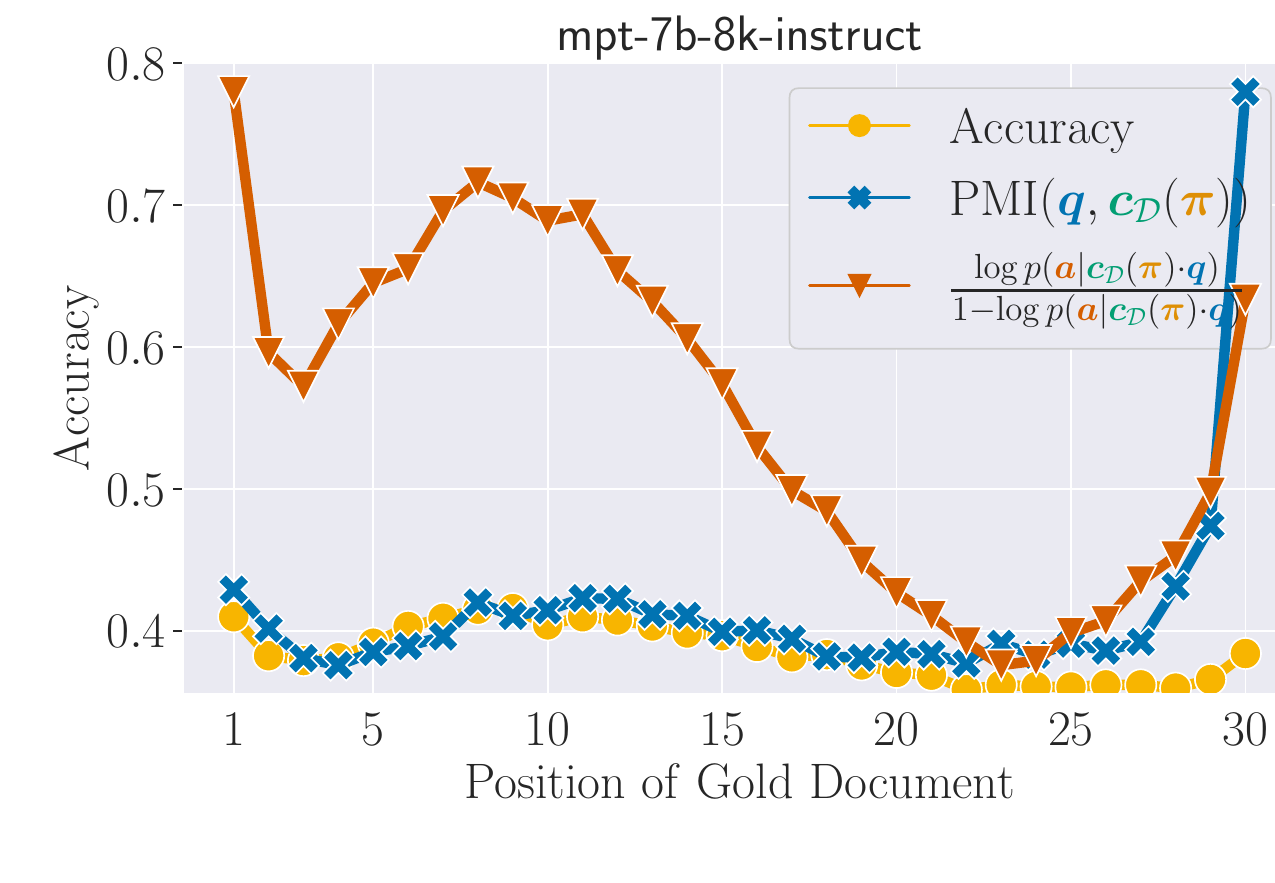}
        \caption{}
    \end{subfigure}
    \caption{QA accuracy, PMI, and log odds ratio of answer likelihood on 30 docs.}
    \label{fig:30docs-llama-3-all-1}
\end{figure*}

\section{Synthetic Experiment}

\begin{figure}[!t]
\begin{tcolorbox} \tt \small
\begin{lstlisting}[language=Python,basicstyle=\tiny]
{
"ebad6435-1e86-4b9e-836a-9a88a8c93743": 
    "c13ac8fc-81fe-408a-bf8f-914b6b8dc310",
"33e652a0-fbcd-4abd-9935-14043ef82de9": 
    "339ffb66-ec38-4d2a-a99f-67755d87eec3",
"7a990232-7ddd-41b6-a8eb-1c61dc96da3c": 
    "0d233f17-9d85-441e-868c-aa682d3dbbe7",
...
}
Key: "7a990232-7ddd-41b6-a8eb-1c61dc96da3c"
Value: "0d233f17-9d85-441e-868c-aa682d3dbbe7"
\end{lstlisting}
\end{tcolorbox}
\caption{Example input for key--value retrieval task.}\label{fig:kv-retrieval-example}
\end{figure}

\begin{table}[]
    \centering \small
    \begin{tabular}{ccc} \toprule
        Key Location   & $\lmprefix(\question \mid \context)$  & $\lmprefix( \answer \mid \context \bcdot \question)$ \\ \midrule
        0   & -3.96 & -0.76 \\
        34  & -6.60 & -0.87 \\
        69  & -7.87 & -0.82 \\
        104 & -8.70 & -1.08 \\
        139 & -8.03 & -0.76 \\ \bottomrule
    \end{tabular}
    \caption{Question likelihood and answer likelihood on synthetic key--value retrieval task using Llama-3.1-8B-Instruct model.}
    \label{tab:llama-3.1-8b-inst-kv}
\end{table}
\begin{table}[]
    \centering \small
    \begin{tabular}{ccc} \toprule
        Key Location   & $\lmprefix(\question \mid \context)$  & $\lmprefix( \answer \mid \context \bcdot \question)$ \\ \midrule
        0   & -3.01 & -0.08 \\
        34  & -6.22 & -0.15 \\
        69  & -6.86 & -0.31 \\
        104 & -7.87 & -0.27 \\
        139 & -7.33 & -0.07 \\ \bottomrule
    \end{tabular}
    \caption{Question likelihood and answer likelihood on synthetic key--value retrieval task using Llama-3.1-8B model.}
    \label{tab:tab:llama-3.1-8b-kv}
\end{table}
\begin{table}[]
    \centering \small
    \begin{tabular}{ccc} \toprule
        Key Location   & $\lmprefix(\question \mid \context)$  & $\lmprefix( \answer \mid \context \bcdot \question)$ \\ \midrule
        0   & -4.03 & -0.00 \\
        34  & -6.31 & -0.12 \\
        69  & -8.15 & -0.23 \\
        104 & -9.67 & -0.15 \\
        139 & -8.77 & -0.04 \\ \bottomrule
    \end{tabular}
    \caption{Question likelihood and answer likelihood on synthetic key--value retrieval task using Mistral-7B-Instruct model.}
    \label{tab:mistral-7b-inst-kv}
\end{table}

\onecolumn

\section{Proof of \Cref{prop:distribution}} \label{app:proof}
\propdistribution*
\begin{proof} 
First note that, by Bayes' rule, we have
\begin{equation}
    \lmprefix(\answer \mid  \contextDpi  \bcdot \question) = \frac{\lmprefix(\centerBox \answer \mid \contextDpi) \lmprefix(\question \mid \contextDpi \centerBox \answer)}{ \lmprefix( \contextDpi  \bcdot \question)}.
\end{equation}
Then, 
\begin{subequations}
\begin{align}
      \log &\frac{\lmprefix(\answer \mid \question \bcdot \contextDpi)}{1-\lmprefix(\answer \mid \question \bcdot \contextDpi)} =    \log \frac{\lmprefix(\answer \mid \question \bcdot \contextDpi)}{\sum_{\baranswer \in \kleene{\alphabet}} \mathbbm{1}\{\baranswer \not\preceq \answer\} \lmprefix(\baranswer \mid \question \bcdot \contextDpi)} \\
       &=    \log \frac{\frac{\lmprefix(\centerBox \answer \mid \contextDpi) \lmprefix(\question \mid \contextDpi \centerBox \answer)}{ \lmprefix( \contextDpi  \bcdot \question)}}{\sum_{\baranswer \in \kleene{\alphabet}} \mathbbm{1}\{\baranswer \not\preceq \answer\} \frac{\lmprefix(\centerBox \baranswer \mid \contextDpi) \lmprefix(\question \mid \contextDpi \centerBox \baranswer)}{\lmprefix(\contextDpi  \bcdot \question)}} \qquad \qquad \qquad\qquad \qquad \text{\color{gray} (Bayes' rule)} \\
      &=    \log \frac{\lmprefix(\centerBox \answer \mid \contextDpi) \lmprefix(\question \mid \contextDpi \centerBox \answer)}{\sum_{\baranswer \in \kleene{\alphabet}} \mathbbm{1}\{\baranswer \not\preceq \answer\} \lmprefix(\centerBox \baranswer \mid \contextDpi) \lmprefix(\question \mid \contextDpi \centerBox \baranswer)} \\
     &=    \log \frac{\lmprefix(\centerBox \answer \mid \contextDpi) \lmprefix(\question \mid \contextDpi)}{\left(\sum_{\baranswer \in \kleene{\alphabet}} \mathbbm{1}\{\baranswer \not\preceq \answer\} \lmprefix(\centerBox \baranswer \mid \contextDpi)\right) \lmprefix(\question)} \qquad \qquad\qquad \qquad\qquad  \text{\color{gray} (\Cref{ass:assumption})} \\
     &=    \log \frac{\lmprefix(\centerBox \answer \mid \contextDpi) }{\sum_{\baranswer \in \kleene{\alphabet}} \mathbbm{1}\{\baranswer \not\preceq \answer\} \lmprefix(\centerBox \baranswer \mid \contextDpi)}\frac{\lmprefix(\question \mid \contextDpi)}{\lmprefix(\question)} \\
     &=    \underbrace{\log \frac{\lmprefix(\centerBox \answer \mid \contextDpi) }{\sum_{\baranswer \in \kleene{\alphabet}} \mathbbm{1}\{\baranswer \not\preceq \answer\} \lmprefix(\centerBox \baranswer \mid \contextDpi)}}_{\defequals C(\answer, \contextDpi)} + \log \frac{\lmprefix(\question \mid \contextDpi)}{\lmprefix(\question)} \\
    &=  \pmi(\question,\contextDpi) + C(\answer, \contextDpi),
\end{align}
\end{subequations}
where $C(\answer, \contextDpi)$ is constant with respect to $\question$.
\end{proof}

\end{document}